\documentclass[a4paper, 11pt]{article}
\usepackage[mono=false]{libertine}
\usepackage[top=2.5cm, bottom=2.5cm, left=2cm, right=2cm]{geometry}
\usepackage[utf8]{inputenc}
\usepackage{algorithm}
\usepackage[T1]{fontenc}
\usepackage{doi}
\usepackage{algpseudocode}
\usepackage{amsfonts}       
\usepackage{amsmath,amsthm,amssymb}
\usepackage{authblk}
\usepackage{bbm}
\usepackage{multirow}
\usepackage{graphicx}     
\usepackage{subcaption}

\usepackage[utf8]{inputenc} 
\usepackage[T1]{fontenc}    
\usepackage{hyperref}
\hypersetup{ colorlinks, breaklinks=true, urlcolor=orange, linkcolor=blue, citecolor=blue }

\usepackage[capitalise]{cleveref}
\crefname{equation}{Equation}{Equations}
\crefname{figure}{Figure}{Figures}
\crefname{table}{Table}{Tables}
    
\usepackage{url}            
\usepackage{booktabs}       
\usepackage{microtype}      
\usepackage{nicefrac}       
\usepackage{xcolor}         
\usepackage{xfrac}
\usepackage{multirow}

\usepackage{natbib}
\newtheorem{assumption}{Assumption}
\newtheorem{theorem}{Theorem}

\newtheorem{lemma}{Lemma}

\newtheorem{remark}{Remark}

\DeclareMathOperator{\prox}{prox}
\DeclareMathOperator{\sign}{sign}

\DeclareMathOperator*{\argmin}{arg\,min}
\DeclareMathOperator*{\var}{var}
\DeclareMathOperator*{\polylog}{polylog}

\title{Accelerating Conformal Prediction via Approximate Leave-One-Out}

\author[1]{%
Jiachen Cong,
Jingbo Liu\thanks{%
Emails: \texttt{jcong3@illinois.edu},
\texttt{jingbol@illinois.edu}. This research was supported in part by NSF Grant DMS-2515510.
}%
}

\affil[1]{
Department of Statistics,
  University of Illinois Urbana-Champaign,
  Champaign, IL, 61820, USA
}
\begin{document}

\maketitle

\begin{abstract}   
    While conformal prediction provides a general framework for uncertainty quantification in predictive inference,  its application is often limited by computational cost. Recent methods, including Jackknife+ and Jackknife-minmax, achieve faster computation by trading a slight loss of efficiency relative to full conformal prediction, but still requires computing leave-one-out refits for all observations. In this paper, we further accelerate conformal prediction by incorporating approximate leave-one-out (ALO) estimators, and establish asymptotic coverage and efficiency. While our proof draws on methods developed for analyzing the consistency of ALO cross-validation risk estimators in high-dimensional statistics, it requires adaptations to handle conformal prediction, where leave-$i$-out residuals are needed for predictions at $x_{n+1}$ rather than just at the training covariate $x_i$. Simulation results validate our theoretical findings, showing that the ALO-based methods achieve coverage and efficiency comparable to the exact methods, while significantly reducing the runtime.
\end{abstract}
\section{Introduction}
\noindent

Quantifying uncertainty is a central problem in statistics, often addressed by constructing prediction intervals for a new response $y_{n+1}$ given training data $\mathcal{D}=\{(\boldsymbol{x}_i,y_i)\}_{i=1}^n$ and $\boldsymbol{x}_{n+1}$. Conformal prediction methods provide distribution-free guarantees, with Full Conformal Prediction (FCP) \cite{vovk2005} being a canonical approach. However, FCP is computationally expensive due to its requirement of searching over candidate response values. In contrast, Jackknife+ \cite{jackknife+} directly output prediction intervals and are more suitable for continuous-response settings, but still incur a computational cost scaling linearly with $n$. In this work, we focus on Jackknife+ and Jackknife-minmax as practical alternatives to FCP.

To address this limitation, recent work has explored ALO techniques. For example, \cite{CP-AMP} proposed using AMP-based estimator to accelerate FCP; \cite{Bhatt_FastCP,AppFCP_via_IF,Prinster2023EfficientAP} proposed influence function-based approach for constructing ALO estimators. Another line of work \cite{ElKaroui2013,ElKaroui2018} proves ALO error bounds and use them to show central limit theorems of the estimator, but the latter requires stronger assumptions. 

However, existing approaches lack rigorous error control for ALO in this setting, which may lead to suboptimal efficiency in practice. Moreover, \cite{CP-AMP} assumes i.i.d.\ covariates with diagonal covariance, whereas we consider a more general Gaussian model $\mathcal{N}(0,\Sigma)$ with $\Sigma$ (not necessarily diagonal).

In this paper, we develop accelerated versions of Jackknife+ and Jackknife-minmax by incorporating ALO estimators derived via Newton updates and the Woodbury identity. Our approach significantly reduces computational cost while preserving the statistical guarantees of the original methods.

{\bf Contribution.}
\begin{itemize}
    \item We establish uniform error bounds for ALO-based predictions with a new covariate $x_{n+1}$, showing that the approximation error
\begin{align}\label{approximation_error}
\bigl|\boldsymbol{x}_{n+1}^\top\tilde{\theta}_{/i} - \boldsymbol{x}_{n+1}^\top\hat{\theta}_{/i}\bigr|
\end{align}
remains uniformly controlled, whereas prior work \citep{wangzhoumaleki2018,RadMaleki2020,Maleki2023} focused only on approximation for $\boldsymbol{x}_i$, $i\in \{1,2,\ldots,n\}$. Our proof leverages some decomposition techniques to decompose (\ref{approximation_error}) to terms can be controlled, thereby reducing the new-covariate case to the setting studied in \citep{RadMaleki2020}. Compared with \citep{CP-AMP}, we weaken the assumption on the covariance structure of the distribution of $\boldsymbol{x}_i$; in particular, we no longer require the covariance matrix to be diagonal.
    \item We propose the accelerated Jackknife+ and Jackknife-minmax methods for constructing prediction intervals based on ALO estimator $\tilde{\theta}_{/i}$, and prove they asymptotically retain the same coverage probability and efficiency (interval length) as their original counterparts.
    \item In a special high-dimensional linear model, we establish that \textit{Full Conformal}, \textit{Split Conformal}, \textit{Jackknife+}, and \textit{Jackknife-minmax} prediction intervals are asymptotically equivalent.
\end{itemize}

{\bf Related work.}
A substantial body of work has examined the construction and error bounds of ALO estimators. \citep{Beirami2017,wangzhoumaleki2018,RadMalekiZhou2020,RadMaleki2020,XuMaleki2021,Maleki2023,Maleki2025} develop explicit forms of ALO estimators using the Woodbury identity and analyze their theoretical and empirical performance in estimating out-of-sample risk. In particular, \citep{RadMaleki2020,wangzhoumaleki2018,Maleki2025} investigate ALO estimators under smoothed regularizers, covering $\ell_1$ penalties, non-smooth but piecewise twice-differentiable regularizers, and more general non-smooth convex Lipschitz regularizers.
Another line of work \citep{ElKaroui2013,ElKaroui2018} prove ALO error bounds and use them to show central limit theorems of the estimator, but the latter requires stronger assumptions.

{\bf Notations.}
Consider a model parametrized by $\theta\in\mathbb{R}^p$ trained on data $\mathcal{D}$. Let $r(\theta)=(1 - \eta) r_0(\theta) + \eta \theta^{\top} \theta$ where $\eta\in(0,1)$ and $r_0$ is a smooth non-negative convex regularizer. Let $\boldsymbol{x}_i^\top \in \mathbb{R}^{1 \times p}$  stand for the $i$-th row of 
$X:=(\boldsymbol{x}_1,\boldsymbol{x}_2,...,\boldsymbol{x}_n)^\top\in \mathbb{R}^{n \times p}$; and we define $\mathbf{y}:=(y_1,y_2,...,y_n)^\top$. 
$\mathbf{y}_{/i} \in \mathbb{R}^{(n-1)\times 1}$ and 
$X_{/i}\in \mathbb{R}^{(n-1)\times p}$ stand for $\mathbf{y}$ and 
$X$, excluding the $i$th entry $y_i$ and the $i$th row $x_i^\top$, respectively. We assume that 
\begin{align}\label{optimization_1}
    \hat{\theta}:=\arg\min_{\theta\in\mathbb{R}^p}\left\{\sum_{j=1}^n\ell(y_j,\boldsymbol{x}_j^\top\theta)+\lambda r(\theta)\right\}
\end{align}
is the empirical risk minimizer~(ERM), where the loss function $\ell(y_j,\boldsymbol{x}_j^\top\theta)$ and regularizer $r(\theta)$ are twice-differentiable in $\theta$, and $\ell(y_j,\boldsymbol{x}_j^\top\theta)$ is convex in $\theta$. Similarly, we define the leave-$i$-out ERM estimator by 
\begin{align}\label{exact}
    \hat{\theta}_{/i}:=\arg\min_{\theta\in\mathbb{R}^p}\left\{\sum_{j\neq i}\ell(y_j,\boldsymbol{x}_j^\top\theta)+\lambda r(\theta)\right\}.
\end{align}
Define
\[
\dot{\ell}_i(\theta) := \left. \frac{\partial \ell(y_i, z)}{\partial z} \right|_{z = \boldsymbol{x}_i^\top \theta}, \ddot{\ell}_i(\theta) := \left. \frac{\partial^2 \ell(y_i, z)}{\partial z^2} \right|_{z = \boldsymbol{x}_i^\top \theta}.
\]
Similarly, we define $\ddot{r}(\theta)=\nabla^2_\theta r(\theta)$. 

Define $f(\boldsymbol{x}^\top\theta)$ as the prediction model $f:\mathbb{R}\to\mathbb{R}$, with parameter $\theta$. Polynomials of $\log(n)$ are denoted by $\mathrm{polylog}(n)$. We define $\hat{q}_{n,\alpha}^+\{d_i\}$ as the $\lceil (1 - \alpha)(n + 1) \rceil$-th smallest value of $d_1, \ldots, d_n$, and $\hat{q}_{n, \alpha}^- \{d_i\}$ as the $\lfloor \alpha(n + 1) \rfloor$-th smallest value of $d_1, \ldots, d_n$. We use $o(1)$ to denote a deterministic sequence that vanishes as $n \to \infty$, and $o_{\mathbb{P}}(1)$ to denote a sequence of random variables that converges to zero in probability.

\section{Preliminaries}
\subsection{Full conformal prediction}
Let $\mathcal{D}_n=\{(\boldsymbol{x}_i,y_i)\}_{i=1}^{n}$ be the training data, and
let $\boldsymbol{x}_{n+1}$ denote a test covariate for which the response
$y_{n+1}$ is unobserved. Full conformal prediction (FCP) outputs a prediction set
$\hat{C}^{\mathrm{FCP}}_{\alpha}(\boldsymbol{x}_{n+1})$ satisfying
\begin{equation}
\mathbb{P}\bigl\{y_{n+1}\in
\hat{C}^{\mathrm{FCP}}_{\alpha}(\boldsymbol{x}_{n+1})\bigr\}
\geq 1-\alpha,
\end{equation}
for any user-specified miscoverage level $\alpha\in(0,1)$, under the
exchangeability of
$(\boldsymbol{x}_1,y_1),\ldots,(\boldsymbol{x}_n,y_n),
(\boldsymbol{x}_{n+1},y_{n+1})$.

To construct the prediction set, FCP considers each candidate value
$y\in\mathbb{R}$ as a hypothetical response for the test point and refits the
model on the augmented dataset
\[
\mathcal{D}_n^{(y)}
=
\mathcal{D}_n\cup\{(\boldsymbol{x}_{n+1},y)\}.
\]
Let $\hat{\theta}(y)$ denote the estimator obtained from $\mathcal{D}_n^{(y)}$.
Define the corresponding nonconformity scores by
\[
\sigma_i(y)
:=
\bigl|f(\boldsymbol{x}_i^\top\hat{\theta}(y))-y_i\bigr|,
\qquad i=1,\ldots,n,
\]
and
\[
\sigma_{n+1}(y)
:=
\bigl|f(\boldsymbol{x}_{n+1}^\top\hat{\theta}(y))-y\bigr|.
\]
Let $\sigma_{(1)}(y)\leq\cdots\leq\sigma_{(n+1)}(y)$ denote the order statistics
of $\{\sigma_i(y)\}_{i=1}^{n+1}$. Then the full conformal prediction set is
given by
\begin{equation}
\hat{C}(\boldsymbol{x}_{n+1})
=
\left\{
y\in\mathbb{R}:
\sigma_{n+1}(y)
\leq
\sigma_{(\lceil(1-\alpha)(n+1)\rceil)}(y)
\right\}.
\label{e_full}
\end{equation}
Equivalently, FCP includes a candidate value $y$ if the residual of the
hypothetical test point conforms with the residuals computed from the augmented
sample. This method is distribution-free and enjoys finite-sample coverage, but
it is computationally expensive because $\hat{\theta}(y)$ must be recomputed for
each candidate value $y$, or for each value on a fine grid in practice; see
\citep{vovk2005,Angelopoulos2022}.

\subsection{Jackknife+}
\noindent

Specifically, we consider two variants of FCP in our paper, called Jackknife+ and Jackknife-minmax \cite{jackknife+}. The construction of prediction intervals is introduced in (\ref{orginal_conclusion}) and (\ref{original_conclusion_mm}), respectively. Let $R_i^{LOO}:=|y_i-f(\boldsymbol{x}_i^\top\hat{\theta}_{/i})|$, and define
\begin{align}\label{orginal_conclusion}
\hat{C}_{n,\alpha}^{\text{jackknife+}}
&:= \Big[
\hat{q}_{n,\alpha}^{-}
\big\{
f(\boldsymbol{x}_{n+1}^\top \hat{\theta}_{/i})
- R_i^{\text{LOO}}
\big\},
\hat{q}_{n,\alpha}^{+}
\big\{
f(\boldsymbol{x}_{n+1}^\top \hat{\theta}_{/i})
+ R_i^{\text{LOO}}
\big\}
\Big], 
\end{align}
\begin{align}\label{original_conclusion_mm}
\hat{C}_{n,\alpha}^{\text{jack-mm}}
&:= \Big[
\min_{i=1,\dots,n}
f(\boldsymbol{x}_{n+1}^\top \hat{\theta}_{/i})
- \hat{q}_{n,\alpha}^{+}
\big\{ R_i^{\text{LOO}} \big\},
\max_{i=1,\dots,n}
f(\boldsymbol{x}_{n+1}^\top \hat{\theta}_{/i})
+ \hat{q}_{n,\alpha}^{+}
\big\{ R_i^{\text{LOO}} \big\}
\Big],
\end{align}

\par
In \cite{jackknife+}, the authors established following results:
\begin{align}
\mathbb{P}\{y_{n+1} \in \hat{C}_{n,\alpha}^{\text{jackknife+}}\}
&\geq 1 - 2\alpha, 
\label{eq:jack_plus} \\
\mathbb{P}\{y_{n+1} \in \hat{C}_{n,\alpha}^{\text{jack-mm}}\}
&\geq 1 - \alpha.
\label{eq:jack_mm}
\end{align}

\subsection{Leave-one-out approximation}
Following \cite{RadMaleki2020}, 
using the Newton step and the Woodbury identity, the approximate leave-$i$-out estimator $\tilde{\theta}_{/i}\in\mathbb{R}^{p\times1}$ is defined as follows
\begin{align}\label{apps_est}
    \tilde{\theta}_{/i}
:= \hat{\theta} 
+ \frac{\mathbf{J}^{-1} \boldsymbol{x}_i \, \dot{\ell}_i( \hat{\theta})}
{1 - \boldsymbol{x}_i^\top \mathbf{J}^{-1} \boldsymbol{x}_i \, \ddot{\ell}_i(\hat{\theta})},
\end{align}
where $\mathbf{J}:=\left( \sum_{j=1}^n \boldsymbol{x}_j \boldsymbol{x}_j^\top \ddot{\ell}_j(\hat{\theta}) + \lambda \, \mathrm{diag}\left[ \ddot{r}(\hat{\theta}) \right] \right)$.

\subsection{Assumptions in \citep{RadMaleki2020} and \citep{Maleki2025}}

\begin{assumption}\label{assumption1}
(a) $r_0$ is a non-negative, convex and twice differentiable function.
\\
(b) Loss function $\ell(y,z)$ is non-negative, convex and continuously differentiable with respect to $z$.
\end{assumption}
\begin{remark}
    In our main analysis, we assume that $r_0$ is twice differentiable so that the Hessian is well-defined. For certain non-smooth regularizers, such as the $\ell_1$ penalty, smoothing approximations can be employed as in \citep{Maleki2025,RadMaleki2020} to construct ALO estimators with smoothed regularizers. While extending our results to this setting is possible, it lies beyond the main scope of the present paper.
\end{remark}

\begin{assumption}\label{assumption3}
(a) $X = (\boldsymbol{x}_1, \cdots, \boldsymbol{x}_n)^\top$ where $\boldsymbol{x}_i \in \mathbb{R}^p$ are i.i.d.\ $N(0,\Sigma)$ samples. Let $\rho_{\text{max}}$ denote the largest eigenvalue of $\Sigma$. Moreover, there exist constants $0 < c_X \leq C_X$ such that 
  \[
  p^{-1} c_X \leq \sigma_{\min}(\Sigma) \leq \sigma_{\max}(\Sigma) \leq p^{-1} C_X.
  \]
  (b) We assume observations are independent and identically
distributed draws from some unknown joint distribution $p(y_i|\boldsymbol{x}_i^\top\theta^*)q(\boldsymbol{x}_i)$, where $\theta^*\in\mathbb{R}^p$ represents the true parameter. 
\end{assumption}
\begin{remark}
We consider the regime where $n,p\to\infty$, $n/p=\delta_0\in(0,\infty)$, and elements of $\theta^*$ are $O(1)$; we have $\|\theta^*\|=O(\sqrt{p})$, and hence $\mathbb{E}[(\boldsymbol{x}^\top\theta^*)^2]=O(1)$ \citep{Maleki2025}.
\end{remark} 

\begin{assumption}\label{assumption4}
(a) We define
  \[
  \dot{\boldsymbol{\ell}}(\cdot):=(\dot{\ell}_1(\cdot), \ldots, \dot{\ell}_n(\cdot))^\top;
  \]
  \[
   \dot{\boldsymbol{\ell}}_{/i}(\cdot) := (\dot{\ell}_1(\cdot), \ldots, \dot{\ell}_{i-1}(\cdot), 
   \dot{\ell}_{i+1}(\cdot), \ldots, \dot{\ell}_n(\cdot))^{\top};
   \]
   \[
   \ddot{\boldsymbol{\ell}}_{/i}(\cdot) := (\ddot{\ell}_1(\cdot), \ldots, \ddot{\ell}_{i-1}(\cdot), 
   \ddot{\ell}_{i+1}(\cdot), \ldots, \ddot{\ell}_n(\cdot))^{\top}.
   \]
   We assume that $c_1(n)=O(\text{polylog(n)})$ and $c_2(n)=O(\text{polylog(n)})$, and $q_n \to 0$ are all functions of $n$, such that with probability at least $1-q_n$ for all $i=1,\ldots,n$
   \begin{align*}
   c_1(n) &> \|\dot{\boldsymbol{\ell}}(\hat{\theta})\|_\infty; \\
   c_2(n) &> \sup_{t \in [0,1]} 
   \frac{\|\ddot{\boldsymbol{\ell}}_{/i}\{(1-t)\hat{\theta}_{/i}+t\hat{\theta}\} - \ddot{\boldsymbol{\ell}}_{/i}(\hat{\theta})\|_2}
   {\|\hat{\theta}_{/i}-\hat{\theta}\|_2};  \\
   c_2(n) &> \sup_{t \in [0,1]} 
   \frac{\|\ddot{r}\{(1-t)\hat{\theta}_{/i}+t\hat{\theta}\} - \ddot{r}(\hat{\theta})\|_2}
   {\|\hat{\theta}_{/i}-\hat{\theta}\|_2}. 
   \end{align*}\\
   (b) There exists a constant $\nu>0$ and a sequence $\tilde{q}_n\to0$ such that for all $i= 1,...,n$
   \begin{align*}
   \inf_{t \in [0,1]} \;
   &\sigma_{\min}\!\Bigg(
   \lambda \, \mathrm{diag}\!\big[\ddot{r}(t\hat{\theta} + (1-t)\hat{\theta}_{/i})\big]+ X_{/i}^\top \, \mathrm{diag}\!\big[\ddot{\boldsymbol{\ell}}_{/i}(t\hat{\theta} + (1-t)\hat{\theta}_{/i})\big] \, X_{/i}
   \Bigg)
   \;\geq\; \nu
\end{align*}
with probability at least $1-\tilde{q}_n$.
\end{assumption}
\begin{remark}
For part (a) of the assumption, \citep[Section~4]{RadMaleki2020} provided its justification for the ridge, elastic net, logistic regression, robust regression, and Poisson regression by choosing $c_1(n)$ and $c_2(n)$ to be polynomials of $n$.
For part (b) of the assumption, since we chose $r(\theta)=(1 - \eta) r_0(\theta) + \eta \theta^{\top} \theta$, we have $\nu=2\lambda\eta$.  
\end{remark}

\subsection{Approximation error bound in \citep{RadMaleki2020}}\label{section2.4}
\par
Under Assumption~\ref{assumption1} to Assumption~\ref{assumption4}, \citep{RadMaleki2020} proved that
\[
\max_{1 \leq i \leq n} \left| x_i^\top \hat{\theta}_{/i} - x_i^\top \hat{\theta} 
- \left( \frac{\dot{\ell}_i(\hat{\theta})}{\ddot{\ell}_i(\hat{\theta})} \right) 
\left( \frac{H_{ii}}{1-H_{ii}} \right) \right|
\leq \frac{C_0}{\sqrt{p}}
\]
holds with probability at least $1 - 4ne^{-p} - \frac{8n}{p^3} - \frac{8n}{(n-1)^3} - q_n - \tilde{q}_n$, where $C_0$ is defined as same as \eqref{c0}.
This result however is not sufficient for our application to conformal prediction, and we will adapt the proof technique to bound the approximation error of $x_{n+1}^\top \hat{\theta}_{/i} $.
\subsection{Lemma~25 in \citep{RadMaleki2020}}\label{lemma_25_maleki}
    \begin{lemma}[Lemma 25 in \citep{RadMaleki2020}]
    Assume that $X^\top(D+\Gamma)X$ and $X^\top D X$ are positive definite, and define:
\[
\Gamma :=\mathrm{diag}(\gamma), \quad 
\bar{\omega}_{\max} := \sigma_{\max}(X X^\top), \quad 
\nu_{\min} := \sigma_{\min}(X^\top(D+\Gamma)X), \quad 
A := X^\top D X.
\]
Then,
\begin{align}\label{lemma1.a}
 \left|z^\top \big(X^\top(D+\Gamma)X\big)^{-1} z
- z^\top\big(X^\top D X\big)^{-1} z \right|
\;\;\le\;\; 
\Big( \|\gamma\|_2 + \frac{\bar{\omega}_{\max}}{\nu_{\min}} \|\gamma\|_4^2 \Big)\,
\big\|X A^{-1} z\big\|_4^2 .
\end{align}
\end{lemma}
\section{Main results}
\begin{assumption}\label{assumption5}
Under assumptions~\ref{assumption1} to \ref{assumption4}, there exists a large enough constant $L<\infty$ such that $\mathbb{P}\{y\in[a,b]|\boldsymbol{x}^\top\theta\}\leq L\cdot|b-a|$ and $|f(b)-f(a)|\leq L\cdot|b-a|$ for $\forall a,b\in\mathbb{R}$.
\end{assumption}
Compared with Assumptions~\ref{assumption1}-\ref{assumption4}, the newly added Assumption~\ref{assumption5} is used to control the effect of the deviation of approximate intervals from the exact ones 
on the resulting coverage guarantee. 
For simplicity, under Assumption~\ref{assumption4}, we define
\begin{align}
&p_1(n):=\left(\frac{16n}{(n-1)^3}+\frac{16n}{p^3}+8ne^{-p}\right)+q_n+\tilde{q}_n \label{def_p1},\\
&p_2(n):=\left(\frac{8n}{(n-1)^3}+\frac{8n}{p^3}+4ne^{-p}\right)+q_n+\tilde{q}_n \label{def_p2},
\end{align}
and when $n/p=\delta_0\in(0,\infty)$ it is clear that
\begin{align}
p_1(n)\xrightarrow{\;\text{$n\to\infty$}\;} 0,\quad  p_2(n)\xrightarrow{\;\text{$n\to\infty$}\;} 0.
\end{align}

\subsection{Algorithm}
\noindent

Motivated by  \eqref{apps_est}, we introduce $\tilde{R}_i^{LOO}:=|y_i-f(\boldsymbol{x}_i^\top\tilde{\theta}_{/i})|$, $i=1,2,\dots,n$. 
Then the jackknife+ prediction interval based on estimators in (\ref{apps_est}) is constructed as in (\ref{ourmethod}),
\begin{equation}\label{ourmethod}
\begin{aligned}
\tilde{C}_{n,\alpha}^{\text{jackknife+}}
&:= \Big[
\hat{q}_{n,\alpha}^{-}
\big\{
f(\boldsymbol{x}_{n+1}^\top \tilde{\theta}_{/i})
- \tilde{R}_i^{\text{LOO}}
\big\},
\hat{q}_{n,\alpha}^{+}
\big\{
f(\boldsymbol{x}_{n+1}^\top \tilde{\theta}_{/i})
+ \tilde{R}_i^{\text{LOO}}
\big\}
\Big].
\end{aligned}
\end{equation}
and the jackknife-minmax prediction interval based on estimators in (\ref{apps_est}) is constructed as in (\ref{ourmethod-mm})
\begin{equation}\label{ourmethod-mm}
\begin{aligned}
\tilde{C}_{n,\alpha}^{\text{jack-mm}}
&:= \Big[
\min_{i=1,\dots,n}
f(\boldsymbol{x}_{n+1}^\top \tilde{\theta}_{/i})
- \hat{q}_{n,\alpha}^{+}
\big\{ \tilde{R}_i^{\text{LOO}} \big\},
\max_{i=1,\dots,n}
f(\boldsymbol{x}_{n+1}^\top \tilde{\theta}_{/i})
+ \hat{q}_{n,\alpha}^{+}
\big\{ \tilde{R}_i^{\text{LOO}} \big\}
\Big].
\end{aligned}
\end{equation}

For simplicity, let 
$\tilde{C}_{n,\alpha}^{\text{jackknife+}}=[\tilde{L}_{jk+},\tilde{U}_{jk+}]$, 
$\hat{C}_{n,\alpha}^{\text{jackknife+}}=[L_{jk+},U_{jk+}]$, 
$\tilde{C}_{n,\alpha}^{\text{jack-mm}}=[\tilde{L}_{jkm}, \tilde{U}_{jkm}]$, 
and $\hat{C}_{n,\alpha}^{\text{jack-mm}}=[L_{jkm},U_{jkm}]$.

\begin{algorithm}[H]
\caption{Accelerating Jackknife+ with ALO Estimators}
\begin{algorithmic}
\Statex \textbf{Input:} $(\boldsymbol{x}_1,y_1),(\boldsymbol{x}_2,y_2),\ldots,(\boldsymbol{x}_n,y_n),\boldsymbol{x}_{n+1}, \alpha$
\Statex \textbf{Output:} Prediction interval for $y_{n+1}$
\State 1. Calculate the full-sample estimator $\hat{\theta}$ (see Eq.~\eqref{optimization_1}).
\State 2. Calculate the approximate leave-$i$-out estimator $\tilde{\theta}_{/i}$ (see Eq.~\eqref{apps_est}) for $i\in\{1,2,...,n\}$.
\State 3. Obtain the approximate $i$-th LO residual $\tilde{R}_i^{LOO}:=|y_i-f(\boldsymbol{x}_i^\top\tilde{\theta}_{/i})|$
and the approximate \\leave-$i$-out estimation $f(\boldsymbol{x}_{n+1}^\top\tilde{\theta}_{/i})$ with $\boldsymbol{x}_{n+1}$ for $i\in\{1,2,...,n\}$.
\State 4. Obtain the prediction interval for $y_{n+1}$, defined as $[\tilde{L}_{jk+},\tilde{U}_{jk+}]$.
\end{algorithmic}
\end{algorithm}
\begin{algorithm}[H]
\caption{Accelerating Jackknife-minmax with ALO Estimators}
\begin{algorithmic}
\Statex \textbf{Input:} $(\boldsymbol{x}_1,y_1),(\boldsymbol{x}_2,y_2),\ldots,(\boldsymbol{x}_n,y_n),\boldsymbol{x}_{n+1}, \alpha$
\Statex \textbf{Output:} Prediction interval for $y_{n+1}$
\State 1. Calculate the full-sample estimator $\hat{\theta}$ (see Eq.~\eqref{optimization_1}).
\State 2. Calculate the approximate leave-$i$-out estimator $\tilde{\theta}_{/i}$ (see Eq.~\eqref{apps_est}) for $i\in\{1,2,...,n\}$.
\State 3. Obtain the approximate $i$-th LO residual $\tilde{R}_i^{LOO}:=|y_i-f(\boldsymbol{x}_i^\top\tilde{\theta}_{/i})|$
and the approximate \\leave-$i$-out estimation $f(\boldsymbol{x}_{n+1}^\top\tilde{\theta}_{/i})$ with $\boldsymbol{x}_{n+1}$ for $i\in\{1,2,...,n\}$.
\State 4. Obtain the prediction interval for $y_{n+1}$, defined as $[\tilde{L}_{jkm},\tilde{U}_{jkm}]$.
\end{algorithmic}
\end{algorithm}

\subsection{Approximation guarantees}
\begin{theorem}\label{Theorem.1}
Let Assumption~\ref{assumption1} to Assumption~\ref{assumption4} hold with $\rho_{\text{max}}=c/p$ and $n/p=\delta_0\in(0,\infty)$. Moreover, suppose that n is large enough such that $q_n+\tilde{q}_n<0.5$. Then with probability at least $1-p_1(n)$, where $p_1(n)$ is defined in (\ref{def_p1}), 
the following bound is valid:
\begin{align}
    \max_{1\leq i\leq n}|\boldsymbol{x}_{n+1}^\top\tilde{\theta}_{/i}-\boldsymbol{x}_{n+1}^\top\hat{\theta}_{/i}|\leq\frac{C_0}{\sqrt{p}},
\end{align}
where
\begin{align}\label{c0}
&C_0 := \; 
\left( \frac{216 c^{3/2}}{\nu^{3}} \right)
\left( 1 + \sqrt{\delta_{0}}(\sqrt{\delta_{0}}+3)^{2} 
       \frac{c \log n}{\log p} \right)\left( c_{1}^{2}(n)c_{2}(n) 
+ c_{1}^{3}(n)c_{2}^{2}(n) \,
       \frac{5\big(c^{1/2} + c^{3/2}(\sqrt{\delta_{0}}+3)^{2}\big)}{\nu^{2}} 
\right).
\end{align}
\end{theorem}
To prove Theorem~\ref{Theorem.2}, we introduce the following lemma.
\begin{lemma}\label{lemma.2}
    Let \(\{a_i\}_{i=1}^n\) be a sorted sequence such that \(a_1 \le a_2 \le \dots \le a_n\), and suppose there exist sequences \(\{b_i\}_{i=1}^n\) satisfying \(|b_i - a_i| \le  \varepsilon\), for some $\varepsilon>0$ and all \(i = 1, \dots, n\). Then for any \(q \in \{1,2,...,n\}\) we have
\[
|b_{(q)}-a_q|\le \epsilon,
\]
where $b_{(q)}$ is the $q$-th smallest value of $\{b_i\}_{i=1}^n$.
\end{lemma}
Proof of the lemma~\ref{lemma.2} is deferred to Section~\ref{prooflemmas} in the appendix.
\begin{theorem}\label{Theorem.2}
Under Assumption~\ref{assumption1} to Assumption~\ref{assumption5}, when $n,p\to\infty$, $n/p=\delta_0\in(0,\infty)$, with probability at least $(1-p_1(n)-p_2(n))$,
\[
|\tilde{U}_{jk+}-U_{jk+}| = o(1),\;|\tilde{L}_{jk+}-L_{jk+}| = o(1);
\]
\[
|\tilde{U}_{jkm}-U_{jkm}| = o(1),\;|\tilde{L}_{jkm}-L_{jkm}| = o(1).
\]
\end{theorem}

\begin{theorem}\label{Theorem.3}
Under Assumption~\ref{assumption1} to Assumption~\ref{assumption5}, when $n,p\to\infty$, $n/p=\delta_0\in(0,\infty)$, we have
\begin{align}
    \mathbb{P}\{y_{n+1}\in\tilde{C}_{n,\alpha}^{jackknife+}\}\geq 1-2\alpha-o(1).
\end{align}
\end{theorem}

\begin{theorem}\label{Theorem.4}
Under Assumption~\ref{assumption1} to Assumption~\ref{assumption5}, when $n,p\to\infty$, $n/p=\delta_0\in(0,\infty)$, we have
\begin{align}
    \mathbb{P}\{y_{n+1}\in\tilde{C}_{n,\alpha}^{jack-mm}\}\geq 1-\alpha-o(1).
\end{align}
\end{theorem}

\section{Numerical experiments}
\subsection{Synthetic data}
In this section, we conduct simulations to support the results in Theorem~\ref{Theorem.3}-\ref{Theorem.4}. We report the mean of coverage, mean of operation time, mean of interval length and mean of Jaccard index of original methods (labeled as ``JK+'' and ``JK-minmax'', where ``JK'' stands for "Jackknife") and accelerated methods (labeled as ``Fast JK+'' and ``Fast JK-minmax''). Recall that the Jaccard index between two sets
$\mathcal{S}_1, \mathcal{S}_2$ is defined as
\begin{equation*}
    \mathcal{J}(\mathcal{S}_1,\mathcal{S}_2) = \frac{|\mathcal{S}_1 \cap \mathcal{S}_2|}{|\mathcal{S}_1 \cup \mathcal{S}_2|} \in [0,1].
\end{equation*}
Values closer to $1$ indicate more precise approximations. In Table~\ref{table3} we report the Jaccard index between Jackknife+ and Fast Jackknife+, Jackknife-minmax and Fast Jackknife-minmax.

In our simulation, we set $\alpha=0.1$, which means our target coverage level is $1-\alpha=0.9$ (in the jackknife+ case, as shown in \citep{jackknife+}, the coverage level is $1-2\alpha=0.8$). We use training sample size $n =100$, test sample size $n_{test}=100$, and repeat the experiment at each dimension $p=50,100,200$, with i.i.d.~data points $(\boldsymbol{x}_i,y_i)$ generated as $\boldsymbol{x}_i\sim \mathcal{N}(0,I_p/\sqrt{p})$ and $y_i|\boldsymbol{x}_i\sim\mathcal{N}(\boldsymbol{x}_i^\top\theta,1)$. The true coefficient vector $\theta$ is randomly generated from a standard normal distribution . For the Ridge regression model, we define the loss function as $\ell(y,\boldsymbol{x}^\top\theta)=(y-\boldsymbol{x}^\top\theta)^2/2$ and the regularization term as $r(\theta)=\frac{1}{2}r_0(\beta)+\frac{1}{2}\theta^{\top} \theta$, where Pseudo-Huber regularizer (this setting allows our simulation results to be compared with those reported in \citep{CP-AMP}) $r_0(\beta)=\sum_{j=1}^p4(\sqrt{1+\frac{\beta_j^2}{4}}-1)$, with the ridge parameter $\lambda=0.1,1$, separately. To obtain stable results, we repeat the procedures above for 50 iterations and report the averaged outcomes.

\begin{table}[H]
\centering
\caption{Comparison of JK+ and Fast JK+.}\label{table1}
\label{tab:jkplus_results}
\scriptsize
\setlength{\tabcolsep}{2.8pt}
\renewcommand{\arraystretch}{1.1}
\begin{tabular}{llrrr}
\toprule
\textbf{Parameters} & \textbf{Model} & \textbf{Coverage} & \textbf{Time(s)} & \textbf{Length} \\
\midrule
(n=100, p=50, \(\lambda\)=1)  & JK+       & 0.880 & 0.075 &  3.808 \\
(n=100, p=100, \(\lambda\)=1) & JK+       & 0.892 & 0.109 &  4.189 \\
(n=100, p=200, \(\lambda\)=1) & JK+       & 0.897 & 0.185 &  4.409 \\
(n=100, p=50, $\lambda$=0.1)  & JK+       & 0.872 & 0.091 &  4.135 \\
(n=100, p=100, $\lambda$=0.1) & JK+       & 0.888 & 0.172 &  4.729 \\
(n=100, p=200, $\lambda$=0.1) & JK+       & 0.889 & 0.354 &  4.662 \\
\midrule
(n=100, p=50, \(\lambda\)=1)  & Fast JK+  & 0.880 & 0.004 &  3.809 \\
(n=100, p=100, \(\lambda\)=1) & Fast JK+  & 0.893 & 0.008 &  4.190 \\
(n=100, p=200, \(\lambda\)=1) & Fast JK+  & 0.897 & 0.043 &  4.410 \\
(n=100, p=50, $\lambda$=0.1)  & Fast JK+  & 0.872 & 0.005 &  4.135 \\
(n=100, p=100, $\lambda$=0.1) & Fast JK+  & 0.888 & 0.009 &  4.734 \\
(n=100, p=200, $\lambda$=0.1) & Fast JK+  & 0.890 & 0.044 &  4.672 \\
\bottomrule
\end{tabular}%
\end{table}
\begin{table}[H]
\centering
\caption{Comparison of JK-minmax and Fast JK-minmax.}\label{table2}
\label{tab:jkminmax_results}
\scriptsize
\setlength{\tabcolsep}{2.8pt}
\renewcommand{\arraystretch}{1.1}
\begin{tabular}{llrrr}
\toprule
\textbf{Parameters} & \textbf{Model} & \textbf{Coverage} & \textbf{Time(s)} & \textbf{Length} \\
\midrule
(n=100, p=50, \(\lambda\)=1)  & JK-minmax      & 0.911 & 0.086 &  4.187 \\
(n=100, p=100, \(\lambda\)=1) & JK-minmax      & 0.920 & 0.103 &  4.549 \\
(n=100, p=200, \(\lambda\)=1) & JK-minmax      & 0.918 & 0.189 &  4.718 \\
(n=100, p=50, $\lambda$=0.1)  & JK-minmax      & 0.926 & 0.094 &  4.925 \\
(n=100, p=100, $\lambda$=0.1) & JK-minmax      & 0.942 & 0.164 &  5.574 \\
(n=100, p=200, $\lambda$=0.1) & JK-minmax      & 0.937 & 0.388 &  5.387 \\
\midrule
(n=100, p=50, \(\lambda\)=1)  & Fast JK-minmax & 0.911 & 0.004 &  4.188 \\
(n=100, p=100, \(\lambda\)=1) & Fast JK-minmax & 0.920 & 0.008 &  4.550 \\
(n=100, p=200, \(\lambda\)=1) & Fast JK-minmax & 0.918 & 0.042 &  4.719 \\
(n=100, p=50, $\lambda$=0.1)  & Fast JK-minmax & 0.926 & 0.004 &  4.925 \\
(n=100, p=100, $\lambda$=0.1) & Fast JK-minmax & 0.943 & 0.008 &  5.750 \\
(n=100, p=200, $\lambda$=0.1) & Fast JK-minmax & 0.937 & 0.047 &  5.400 \\
\bottomrule
\end{tabular}%
\end{table}
\begin{table}[H]
\centering
\caption{Prediction-interval Overlap (Jaccard Index) of Fast JK+ \& JK+ and Fast JK-minmax \& JK-minmax.}\label{table3}
\label{tab:ji_fast_only}
\scriptsize
\setlength{\tabcolsep}{4pt}
\renewcommand{\arraystretch}{1.1}
\begin{tabular}{lrr}
\toprule
\textbf{Parameters} & \textbf{Fast JK+ \& JK+} & \textbf{Fast JK-minmax \& JK-minmax} \\
\midrule
(n=100, p=50, \(\lambda\)=1)   & 0.9997 & 0.9997 \\
(n=100, p=100, \(\lambda\)=1)  & 0.9997 & 0.9996 \\
(n=100, p=200, \(\lambda\)=1)  & 0.9998 & 0.9997 \\
(n=100, p=50, $\lambda$=0.1)   & 0.9996 & 0.9996 \\
(n=100, p=100, $\lambda$=0.1)  & 0.9983 & 0.9982 \\
(n=100, p=200, $\lambda$=0.1)  & 0.9978 & 0.9976 \\
\bottomrule
\end{tabular}%
\end{table}
\par
Table~\ref{table1} and Table~\ref{table2} present the mean of coverage, mean of operation time and mean of interval length of  Jackknife+, Fast Jackknife+, Jackknife-minmax and Fast Jackknife-minmax, respectively. Table~\ref{table3} shows the similarity, measured by Jaccard Index, between the prediction intervals constructed by the accelerated methods and their corresponding original methods. 
\par
We find that our accelerated methods substantially reduce the average computational time while maintaining coverage, and without significantly altering the length (efficiency) of the prediction intervals. The prediction intervals constructed by the accelerated methods exhibit a high degree of similarity to those from the original methods. We can observe that, in most cases, a smaller value of $\lambda$ tends to result in a wider interval, and the interval length also increases as $p$ grows. Moreover, our methods exhibit more significant acceleration when the dimensionality of the covariates is higher.

Compared with the synthetic-data simulation results reported in Table~1 of \citet{CP-AMP}, both of our acceleration methods achieve higher coverage than Taylor-AMP, while also providing more efficient prediction (in terms of shorter average interval length) compared with Taylor-AMP, SCP \citep{vovk2005} and CQP \citep{CQP}. 
With respect to Table~2 of \citet{CP-AMP}, fast Jackknife+ and fast Jackknife-minmax exhibit higher Jaccard similarity to exact LOO than Taylor-AMP and SCP. 
Under Gaussian settings (Table~3 of \citet{CP-AMP}), our methods deliver more efficient prediction than Bayes posterior and FCP combined with Taylor-AMP. 
In comparison with Table~4 of \citet{CP-AMP}, both of our accelerated procedures achieve higher prediction efficiency and higher coverage than Taylor-AMP and approximate homotopy \citep{app_homotopy}, while exhibiting comparable computation time. 
Furthermore, our framework systematically explores multiple dimensions ($p=50,100,200$), corresponding to different $n/p$ ratios, and two regularization strengths ($\lambda=0.1,1$), demonstrating consistent performance across regimes.

These improvements can be attributed to two main factors. 
First, unlike \citet{CP-AMP}, which assumes i.i.d.~features with a diagonal covariance matrix, our framework accommodates a general covariance structure $\Sigma$. This relaxation is more realistic in practice and ensures that the resulting ALO characterization remains accurate even when the features are correlated, thereby preventing the deterioration observed for AMP and Taylor-AMP under non-isotropic designs. Second, while AMP-based methods rely on an ALO heuristic whose approximation error is not rigorously controlled, our approach benefits from a explicit and tighter ALO error bound derived via the Newton step and Woodbury identity. This leads to more accurate leave-one-out predictions, which in turn improves the quality of the prediction intervals and enhances efficiency.

\subsection{Application to Real Data}

We evaluate the proposed methods on the Concrete Compressive Strength Dataset \citep{data_concrete} and the Energy Efficiency Dataset \citep{data_energy}. For both datasets, we consider a linear prediction model
\[
f(\boldsymbol{x}^{\top}\theta)=\boldsymbol{x}^{\top}\theta,
\]
where the full sample estimator is obtained by minimizing the regularized empirical risk objective
\[
\hat{\theta}
=
\arg\min_{\theta\in\mathbb{R}^{p}}
\left\{
\frac{1}{2}\sum_{i=1}^{n}\left(y_i-\boldsymbol{x}_i^{\top}\theta\right)^2
+
\lambda
\left[
2\sum_{j=1}^{p}
\left(
\sqrt{1+\frac{\theta_j^2}{4}}-1
\right)
+
\frac{1}{2}\|\theta\|_2^2
\right]
\right\},
\]
with the regularization parameter fixed at $\lambda=0.1$. For each dataset, the first $80\%$ of the observations are used for training, and the remaining $20\%$ are used for testing.

As shown in Table~\ref{tab:concrete_results} and Table~\ref{tab:energy_results}, the accelerated methods achieve nearly identical empirical coverage and comparable average interval lengths to the original Jackknife+ and Jackknife-minmax methods, while substantially reducing the computational time.

\begin{table}[H]
\centering
\caption{Performance on the Concrete Compressive Strength Dataset}
\label{tab:concrete_results}
\begin{tabular}{lccc}
\toprule
\textbf{Method} & \textbf{Coverage Rate} & \textbf{Operation Time (s)} & \textbf{Average Interval Length} \\
\midrule
Jackknife+              & 0.9660 & 1.7275  & 37.0917 \\
Fast Jackknife+         & 0.9660 & 0.0651 & 37.0851 \\
Jackknife-minmax        & 0.9709 & 1.5735  & 37.8552 \\
Fast Jackknife-minmax   & 0.9709 & 0.0414 & 37.8671 \\
\bottomrule
\end{tabular}
\end{table}

\begin{table}[H]
\centering
\caption{Performance on the Energy Efficiency Dataset}
\label{tab:energy_results}
\begin{tabular}{lccc}
\toprule
\textbf{Method} & \textbf{Coverage Rate} & \textbf{Operation Time (s)} & \textbf{Average Interval Length} \\
\midrule
Jackknife+                & 0.9351 & 1.6567  & 12.0404 \\
Fast Jackknife+           & 0.9351 & 0.0509 & 12.0403 \\
Jackknife-minmax          & 0.9481 & 1.0640  & 12.2182 \\
Fast Jackknife-minmax     & 0.9481 & 0.0371 & 12.2181 \\
\bottomrule
\end{tabular}
\end{table}

\subsection{Code}
The code used to produce the results in Table~\ref{table1}, Table~\ref{table2}, Table~\ref{table3}, Table~\ref{tab:concrete_results} and Table~\ref{tab:energy_results} can be found in the following github repository:
\url{https://github.com/JiachenCong/Accelerating_Conformal_Prediction_via_Approximate_Leave-One-Out.git}. All experiments were run on an Apple M4 laptop with 16 GB of memory.
\section{Discussion}
In this section, we address two natural questions regarding our algorithm's asymptotic performance in high-dimensional settings \cite{ElKaroui2018}:
\begin{description}
\item[Q1:] We applied the leave-one-out approximation formula to accelerate jackknife+ and jackknife-mm, rather than full conformal prediction,
so that rigorous results in \citep{RadMaleki2020,Maleki2023,Maleki2025} can be used to control the approximation error.
Does this lead to worse efficiency (longer intervals)?

\item[Q2:] 
\citep{ElKaroui2013}
used leave-one-out to prove the central limit theorem M-estimators.
Is it possible to further accelerate our algorithm by directly using the central limit theorem result?
\end{description}
In this section, we address these questions by studying the asymptotic performances of full conformal prediction, split conformal prediction, jackknife+, and jackknife-mm in a high-dimensional setting  rigorously studied in \citep{ElKaroui2018},
although some of the conclusions are expected to hold in broader settings.
Following \citep{ElKaroui2018}, we consider the model $y_i = X_i^{\top} \beta_0 + \epsilon_i$, and
\begin{align}
\hat{\beta} = \arg\min_{\beta \in \mathbb{R}^p} 
\frac{1}{n} \sum_{i=1}^n \rho_i\!\left(y_i - X_i^{\top} \beta\right) 
+ \frac{\tau}{2}\|\beta\|^2. 
\label{e_ka}
\end{align}
Similarly, we define $\hat{\beta}_{/i}$ as the estimator calculated without the $i$-th sample.
Note that, following \citep{ElKaroui2018}, we assume the scaling $\|X_i\|_2=\Theta(\sqrt{n})$ and $\|\beta\|_2=\Theta(1)$, which is consistent with the scaling in the preceding sections if we take $X_i:=\sqrt{n}x_i$   and $\beta:=\frac1{\sqrt{n}}\theta$.
Within the general framework in the preceding sections, \eqref{e_ka} can be understood as the special case where the regularizer is ridge, the log-likelihood $\ell(y,z)=\rho(y-z)$ for some $\rho$,
and $y_i=X_i^{\top}\beta+\epsilon_i$ under $p(y_i|X_i^{\top}\beta)$, where $\epsilon_1,\dots,\epsilon_n$ are i.i.d.\ noise.
The main contribution of \citep{ElKaroui2018} was to use leave-one-out (both observation and predictor) to prove the asymptotic distribution of $\hat{\beta}$.
Though more general sufficient conditions can be found in \citep{ElKaroui2018},
for simplicity let us assume the following condition (see \citep[Section~2.1]{ElKaroui2018}):
\begin{description}
\item[C1]  
$p/n$ has a finite nonzero limit.
\item[C2]
$\rho$ is convex, whose first and second derivatives are uniformly bounded: $\|\rho'\|_{\infty}<\infty$, 
$\|\rho''\|_{\infty}<\infty$.
Furthermore, $\sign(\rho'(t))=\sign(t)$ and $\rho(t)\ge \rho(0)=0$ for all $t\in \mathbb{R}$.
\item[C3] 
$X_1,\dots,X_n$ are i.i.d.\ following $\mathcal{N}(0,I_p)$,
or with i.i.d.\ entries with bounded support,
symmetric density,
and unit variance.
\item[C4] 
$\epsilon_1,\dots,\epsilon_n$ are i.i.d.\ following a differentiable, symmetric, unimodal distribution with variance $\sigma_{\epsilon}^2$, whose density $f$ satisfies $\lim_{t\to\infty}tf(t)=0$.
\item[C5] 
$\|\beta_0\|_2$ remains bounded as $n\to\infty$, and 
$\|\beta_0\|_{\infty}=O(n^{-\epsilon})$ for some $\epsilon>1/4$.
\end{description}
We denote by $\prox_{\rho}$ the proximal map:
\begin{align}
\prox_{\rho}(x):=
\argmin_{y\in\mathbb{R}}
\{\rho(y)+\frac1{2}(x-y)^2\}.
\end{align}

Under the conditions above, the general result in Theorem 2.1 in \citep{ElKaroui2018} implies the following:

\begin{theorem} \textup{\citep{ElKaroui2013}}
\label{thm_el}
Let $\delta:=\lim_{n\to\infty} n/p\in(0,\infty)$,
and $b_0:=\lim_{n\to\infty}
\|\beta_0\|_2^2$.
Let $c, r > 0$ be the  solution to the following equations:
\begin{align}
&\mathbb{E}\!\left[(\prox_{c\rho}'(Z)\right] = 1 - \delta^{-1}+\tau, 
\\
&\delta^{-1}\mathbb{E}\!\left[(Z - \prox_{c\rho}(Z))^2\right]
+\tau^2c^2b_0
=
\delta^{-2} r^2 
\end{align}
where we defined the scalar $\epsilon \sim \mathcal{N}(0,1)$ and 
$Z \sim \mathcal{N}(0,r^2) + \epsilon$.
Then we have 
$\lim_{p \to \infty} \|\hat{\beta}-\beta_0\|=r$ and $\lim_{p \to \infty}\var(\|\hat{\beta}-\beta_0\|^2)=0$ in probability.
\end{theorem}

Denote by $U_{\rm J}$ and $U_{\rm mm}$ the upper limits of the conformal prediction intervals for the $(n+1)$-th observation produced by the jackknife+ and jackknife-minmax algorithms. For any given $M>0$, let $U_{\rm f}(M)$ be the upper limit of the discretized interval $\hat{C}(X_{n+1})\cap [-M,M]\cap \{0,\pm\frac1{M},\frac2{M},\dots\}$, where $\hat{C}(X_{n+1})$ is the full conformal interval. Similarly, define $L_{\rm f}(M)$, $L_{\rm J}$, $L_{\rm mm}$ as the lower limits. Under our assumptions, the differences of these intervals vanish as $n, M \to \infty$:
\begin{theorem}\label{thm6}
Define $U:=X_{n+1}^{\top}\hat{\beta} + \sqrt{\sigma_{\epsilon}^2+r^2}z_{\alpha/2}$, where $z_{\alpha/2}$ denotes the $(1-\alpha/2)$ quantile of a standard normal random variable. Let $a_n:=\max\{|U-U_{\rm J}|,\, |U-U_{\rm mm}|\}$ and $b_{n,M}:=\max_{i=1,\dots,n}|U_{\rm f}-U|$. Then in the sense of weak convergence:
\begin{align}
\lim_{n\to\infty}a_n&=0
\label{e_al}
\\
\lim_{M\to\infty }\lim_{n\to\infty} b_{n,M}&=0
\label{e_bl}
\end{align}
The same holds for the lower bounds, with $L:=X_{n+1}^{\top}\hat{\beta} - \sqrt{\sigma_{\epsilon}^2+r^2}z_{\alpha/2}$.
\end{theorem}
The proof of Theorem~\ref{thm6} can be found in the appendix.

This theorem demonstrates that under the idealized assumptions in \cite{ElKaroui2018}, all considered conformal prediction algorithms are asymptotically equivalent, partially answering Q1. 

Regarding Q2, we note that the asymptotic approximation in  Theorem~\ref{thm6} relies on the CLT result, which requires independence of $X_1,\dots,X_n$ and is much stronger than leave-observation-out approximation. Conversely, jackknife+ (and our proposed ALO version) remains valid in the general exchangeable case, which partially answers Q2.
It is an interesting direction for future research to identify settings in which leave-one-out approximations remain asymptotically valid while asymptotic normality fails. A related work \cite{liu2024stability} shows that leave-one-feature-out approximations remain asymptotically accurate for general correlated sub-Gaussian feature vectors, and the techniques developed there may be extendable to the leave-one-observation-out setting.

\vfill
\newpage
\bibliography{reference}
\newpage
\appendix

\section*{\centering\LARGE\textbf{Appendix}}
\renewcommand{\thesection}{S.\arabic{section}}
In this appendix, we provide detailed proofs of the theorems presented in the main paper.

\section*{Organization}
Table~\ref{Appendix_table_1} presents the organization of the appendix.
\begin{table}[h!]
\centering
\caption{Organization of Supplementary Materials}\label{Appendix_table_1}
\begin{tabular}{ll}
\toprule
\textbf{Section}  & \textbf{Purpose} \\
\midrule
Appendix~\hyperref[existing results]{S.1} &  Presents existing results adapted from \citep{RadMaleki2020} for our proof.
\\
Appendix~\hyperref[theorem1appendix]{S.2} & Provides the complete proof of Theorem~\ref{Theorem.1}.\\
Appendix~\hyperref[theorem2appendix]{S.3} & Provides the complete proof of Theorem~\ref{Theorem.2}.\\
Appendix~\hyperref[theorem3appendix]{S.4} & Provides the complete proof of Theorem~\ref{Theorem.3} and Theorem~\ref{Theorem.4}.\\
Appendix~\hyperref[prooflemmas]{S.5} & Provides the complete proof of Lemma~\ref{lemma.2}.\\
Appendix~\hyperref[theorem6appendix]{S.6} & Provides the complete proof of Theorem~\ref{thm6} and relative lemmas.\\
\bottomrule
\end{tabular}
\end{table}

\section{Existing Results}\label{existing results}
In this section, we present existing results used in our proof, which are adapted from \citep{RadMaleki2020}.

\subsection{\citep[Eq.136]{RadMaleki2020}}\label{eq136}
Define $g_{/i}(\theta)=\lambda\dot{r}(\theta)+X_{/i}^\top \dot{\boldsymbol{\ell}}_{/i}(\theta)$.
The leave-one-out estimate, $\hat{\theta}_{/i} = \hat{\theta} + \Delta_{/i}^*$, satisfies $g_{/i}(\Delta_{/i}^*+\hat{\theta}) = 0$. The multivariate mean-value Theorem yields
\begin{equation}
0 = g_{/i}(\hat{\theta} + \Delta_{/i}^*) = g_{/i}(\hat{\theta}) + \left( \int_0^1 \boldsymbol{J}_{/i}(\hat{\theta} + t \Delta_{/i}^*) dt \right) \Delta_{/i}^*, 
\end{equation}
where the Jacobian is
\begin{equation}
\boldsymbol{J}_{/i}(\theta) = \lambda \, \mathrm{diag}[\ddot{r}(\theta)] + X_{/i}^{\top} \, \mathrm{diag}[\ddot{\boldsymbol{\ell}}_{/i}(\theta)] \, X_{/i}.
\end{equation}
Moreover, $\hat{\beta}$ satisfies
\begin{equation}
0 = \lambda \dot{r}(\hat{\theta}) + X^{\top} \dot{\boldsymbol{\ell}}(\hat{\theta}) = g_{/i}(\hat{\theta}) + \dot{\boldsymbol{\ell}}_i(\hat{\theta}) x_i.
\end{equation}
We obtain
\begin{equation}
\dot{\boldsymbol{\ell}}_i(\hat{\theta}) x_i = - \left( \int_0^1 \boldsymbol{J}_{/i}(\hat{\theta} + t \Delta_{/i}^*) dt \right) \Delta_{/i}^*,
\end{equation}
so that
\begin{equation}
\Delta_{/i}^* = - \dot{\boldsymbol{\ell}}_i(\hat{\theta}) \left( \int_0^1 \boldsymbol{J}_{/i}(\hat{\theta} + t \Delta_{/i}^*) dt \right)^{-1} x_i,
\end{equation}
leading to the following inequality
\begin{equation}
\|\Delta_{/i}^*\|_2 \le \left( \frac{|\dot{\ell}_i(\hat{\theta})|}{\nu} \right) \|\boldsymbol{x}_i\|_2.
\end{equation}
By Assumption 3.(b), we obtain the Eq.136 in \citep{RadMaleki2020}
\begin{align}
    \|\Delta_{/i}^*\|_2 \leq\frac{c_1(n)}{\nu} \|\boldsymbol{x}_i\|_2
\end{align}

\subsection{\citep[Lemma 27]{RadMaleki2020}}
\label{lemma27}
\begin{lemma}[Lemma 27 in \citep{RadMaleki2020}]
Let $x \sim \mathcal{N}(0, \Sigma)$ with $\rho_{\max} := \sigma_{\max}(\Sigma)$, where $\Sigma \in \mathbb{R}^{p \times p}$, then
\begin{equation}
\mathbb{P}\left[ \|x\|_4^2 > 2(1 + c) \rho_{\max} \sqrt{p \log p} \right] \le \frac{2}{p^c}.
\end{equation}

Moreover, if 
\begin{align}
\omega_{\max} &:= \sigma_{\max}(XX^{\top}), \notag\\
\nu_{\min} &:= \sigma_{\min}(J),
\end{align}
where $x$ is independent of the symmetric matrix $J \in \mathbb{R}^{p \times p}$ and $X \in \mathbb{R}^{m \times p}$, then
\begin{align}
\mathbb{P}\left[ \|J^{-1}x\|_4^2 > 2(1 + c)\left(\frac{\rho_{\max}}{\nu_{\min}^2}\right) \sqrt{p \log p} \right] &< \frac{2}{p^c}, \label{lemma_27_1}\\
\mathbb{P}\left[ \|XJ^{-1}x\|_4^2 > 2(1 + c)\left(\frac{\rho_{\max}\omega_{\max}}{\nu_{\min}^2}\right) \sqrt{m \log m} \right] &< \frac{2}{m^c}.\label{lemma_27_2}
\end{align}
\end{lemma}
The proof of this lemma can be found on pages 60 and 61 of \citep{RadMaleki2020}.

\subsection{\citep[Lemma 11]{RadMaleki2020}}\label{lemma11}
\begin{lemma}[Lemma 11 in \citep{RadMaleki2020}]
Let $x \sim \mathcal{N}(0, \Sigma)$ with $\rho_{\max} \triangleq \sigma_{\max}(\Sigma)$, where $\Sigma \in \mathbb{R}^{p \times p}$, then
\begin{equation}
\mathbb{P}\left[ \|x\|_2^2 > 5p\rho_{\max} \right] \le e^{-p}.
\end{equation}
Furthermore, if $X \in \mathbb{R}^{n \times p}$ is composed of independently distributed $\mathcal{N}(0, \tfrac{1}{n})$ entries, then
\begin{equation}
\mathbb{P}\left[ \sqrt{\sigma_{\max}(X^{\top} X)} \ge 1 + \sqrt{\frac{p}{n}} + t \right] 
\le e^{-\frac{n t^2}{2}}.
\end{equation}
\end{lemma}
\subsection{\citep[Lemma 12]{RadMaleki2020}}\label{lemma12}
\begin{lemma}[Lemma 12 in \citep{RadMaleki2020}] If
$X \in \mathbb{R}^{n \times p}$ is composed of independently distributed $\mathcal{N}(0, \Sigma)$ rows, with $\rho_{\max} := \sigma_{\max}(\Sigma)$, where $\Sigma \in \mathbb{R}^{p \times p}$, then
\[
\mathbb{P}\left[ \sigma_{\max}(X X^{\top}) \ge (\sqrt{n} + 3\sqrt{p})^2 \rho_{\max} \right] \le e^{-p}.
\]
\end{lemma}
The proof of this lemma can be found on pages 34 and 35 of \citep{RadMaleki2020}.

\subsection{\citep[Theorem 3]{RadMaleki2020}}\label{Theorem3maleki}
\begin{theorem}[Theorem 3 in \citep{RadMaleki2020}]
Let $n/p = \delta_0$ and \citep[Assumption~5]{RadMaleki2020} hold with $\rho_{\max} = c/p$. Moreover, suppose that \citep[Assumption~6 and ~7]{RadMaleki2020} are satisfied, and that $n$ is large enough such that $q_n + \tilde{q}_n < 0.5$. Then with probability at least 
\begin{align}
1 - 4ne^{-p} - \frac{8n}{p^3} - \frac{8n}{(n-1)^3} - q_n - \tilde{q}_n,
\end{align}
the following bound is valid:
\begin{align}
\max_{1 \le i \le n} 
\left|
x_i^\top \hat{\theta}_{/i} - x_i^\top \hat{\theta} -
\left( 
\frac{\dot{\ell}_i(\hat{\theta})}{\ddot{\ell}_i(\hat{\theta})}
\right)
\left(
\frac{H_{ii}}{1 - H_{ii}}
\right)
\right|
\le
\frac{Q_0}{\sqrt{p}},
\end{align}
where
\begin{align}
Q_0 := \; 
\left( \frac{72 c^{3/2}}{\nu^{3}} \right)
\left( 1 + \sqrt{\delta_{0}}(\sqrt{\delta_{0}}+3)^{2} 
       \frac{c \log n}{\log p} \right)\left( c_{1}^{2}(n)c_{2}(n) 
+ c_{1}^{3}(n)c_{2}^{2}(n) \,
       \frac{5\big(c^{1/2} + c^{3/2}(\sqrt{\delta_{0}}+3)^{2}\big)}{\nu^{2}} 
\right).
\end{align}
\end{theorem}

\section{Proof of Theorem~\ref{Theorem.1}}\label{theorem1appendix}
We introduce the following definitions:
\begin{itemize}
    \item (1) $\Delta_{/i}^*:=\hat{\theta}_{/i} -\hat{\theta} $,
    and $\boldsymbol{J}_{/i}:=
\lambda \, \mathrm{diag}\!\big[\ddot{r}_s(\hat{\theta}_{/i})\big]
+ X_{/i}^\top \, \mathrm{diag}\!\big[\ddot{\boldsymbol{\ell}}_{/i}(\hat{\theta}_{/i})\big] \, X_{/i}$, where $X_{/i}=(\boldsymbol{x}_{1},...,\boldsymbol{x}_{i-1},\boldsymbol{x}_{i+1},...,\boldsymbol{x}_{n})^\top$. 
    \item (2) Events $\tilde{E}_i:=\left\{ 
  3\bar{C}_i \left( 
    \left\| X_{/i} \boldsymbol{J}_{/i}^{-1}(\hat{\theta}_{/i} ) \boldsymbol{x}_i \right\|_4^2 
    + 
    \left\| \boldsymbol{J}_{/i}^{-1}(\hat{\theta}_{/i} ) \boldsymbol{x}_{i} \right\|_4^2 
  \right) 
  > C \frac{\log p}{\sqrt{p}} 
\right\}$, \\ and 
$\tilde{E}_{n+1,i}:=\left\{ 
  3\bar{C}_{i} \left( 
    \left\| X_{/i} \boldsymbol{J}_{/i}^{-1}(\hat{\theta}_{/i} ) \boldsymbol{x}_{n+1} \right\|_4^2 
    + 
    \left\| \boldsymbol{J}_{/i}^{-1}(\hat{\theta}_{/i} ) \boldsymbol{x}_{n+1} \right\|_4^2 
  \right) 
  > C \frac{\log p}{\sqrt{p}} 
\right\}$.
\item (3) $\bar{C}_{i}:=4\|\boldsymbol{x}_{i}\|_2 
\left( \frac{c_1^2(n) c_2(n)}{\nu} \right) 
\left( 1 + \frac{2c_1(n) c_2(n)}{\nu^2} (1 + \omega_{\max}) \|\boldsymbol{x}_{i}\|_2 \right)$.
\item (4) Define $\Phi_1:=\boldsymbol{J}_{/i}^{-1}(\hat{\theta}_{/i} )
-  \left( \boldsymbol{J}_{/i}(\hat{\theta}_{/i} ) 
+ \bar{X}_{/i}^\top \, \mathrm{diag}\!\left[\int_0^1 \gamma_{ {-(1-t)\Delta_{/i}^*} \, /i }(\hat{\theta}_{/i} ) dt \right] \bar{X}_{/i}\right)^{-1}$ and \\$\Phi_2:=\boldsymbol{J}_{/i}^{-1}(\hat{\theta}_{/i} )
-  \left( \boldsymbol{J}_{/i}(\hat{\theta}_{/i} )
+ \bar{X}_{/i}^\top \, \mathrm{diag}\!\left[ \gamma_{ {-\Delta_{/i}^*} \, /i } (\hat{\theta}_{/i} ) \right] \bar{X}_{/i} \right)^{-1}$ 
, where for $\delta,\theta\in\mathbb{R}^p$
\[
\gamma_{\delta / i}(\theta):=
\begin{bmatrix}
\ddot{\boldsymbol{\ell}}_{/i}(\theta + \delta) - \ddot{\boldsymbol{\ell}}_{/i}(\theta) \\
\lambda \big( \ddot{r}(\theta + \delta) - \ddot{r}(\theta) \big)
\end{bmatrix},\quad\bar{X}_{/i} :=
\begin{bmatrix}
X_{/i} \\
I
\end{bmatrix}
\in \mathbb{R}^{(n-1+p)\times p}, 
\]
\[
\quad\mathbf{D}_{/i}(\theta) :=
\mathrm{diag}\!\begin{bmatrix}
\ddot{\boldsymbol{\ell}}_{/i}(\theta) \\
\lambda \ddot{r}(\theta)
\end{bmatrix}
\in \mathbb{R}^{(n-1+p)\times(n-1+p)} .
\]
Define $\boldsymbol{J}_{/i}(\theta):=\bar{X}_{/i}^\top \mathbf{D}_{/i}(\theta)\bar{X}_{/i}$.
\end{itemize}

Consider
\begin{align}
&\quad\boldsymbol{x}_{n+1}^\top\tilde{\theta}_{/i} -\boldsymbol{x}_{n+1}^\top\hat{\theta}_{/i} 
\nonumber
\\
&
=\dot{\ell}_i(\hat{\theta }) \boldsymbol{x}_{n+1}^\top 
M\boldsymbol{x}_{i} 
\\
&=\frac1{2}\dot{\ell}_i(\hat{\theta }) (\boldsymbol{x}_i+\boldsymbol{x}_{n+1})^\top 
M
(\boldsymbol{x}_i+\boldsymbol{x}_{n+1})
-
\frac1{2}\dot{\ell}_i(\hat{\theta }) \boldsymbol{x}_{n+1}^\top 
M\boldsymbol{x}_{n+1} 
-\frac1{2}\dot{\ell}_i(\hat{\theta }) \boldsymbol{x}_i^\top 
M\boldsymbol{x}_i
\\
&=A_1-A_2-A_3,
\label{e218}
\end{align}
where we defined 
\begin{align}
M:=\left( \int_0^1 \boldsymbol{J}_{/i}(\hat{\theta}_{/i}  - (1-t)\Delta_{/i}^*) dt \right)^{-1}
- \boldsymbol{J}_{/i}^{-1}(\hat{\theta}_{/i}  - \Delta_{/i}^*). 
\end{align}
Clearly we have
\begin{align}
|\boldsymbol{x}_{n+1}^\top\tilde{\theta}_{/i} -\boldsymbol{x}_{n+1}^\top\hat{\theta}_{/i} |
\le 
|A_1|+|A_2|+|A_3|.
\end{align}

    By Assumption~\ref{assumption4}.(a), we have
    \begin{equation}
        \dot{\ell}_i(\hat{\theta} )=O(\text{polylog}(n)).
    \end{equation}
    By the definition of $\gamma_{\delta/i}(\theta)$, we have
    \begin{align}\label{aux_definition}
        \boldsymbol{J}_{/i}(\theta + \delta)
= \boldsymbol{J}_{/i}(\theta)
+ \bar{X}_{/i}^{\top} 
  \operatorname{diag}\!\big[\gamma_{\delta/i}(\theta)\big]
  \bar{X}_{/i}.
    \end{align}
For the new $A_1$ term, we have
\begin{align}\label{newa1}
    &\left|(\boldsymbol{x}_i+\boldsymbol{x}_{n+1})^\top M(\boldsymbol{x}_i+\boldsymbol{x}_{n+1})\right| \nonumber\\
    =&\left|(\boldsymbol{x}_{i}+\boldsymbol{x}_{n+1})^\top\left[ \left( \int_0^1 \boldsymbol{J}_{/i}(\hat{\theta}_{/i}  - (1-t)\Delta_{/i}^*) dt \right)^{-1} - \boldsymbol{J}_{/i}^{-1}(\hat{\theta}_{/i}  - \Delta_{/i}^*) \right](\boldsymbol{x}_{i}+\boldsymbol{x}_{n+1})\right|\nonumber\\
    =&\left|(\boldsymbol{x}_{i}+\boldsymbol{x}_{n+1})^\top\left[ \left( \int_0^1 \boldsymbol{J}_{/i}(\hat{\theta}_{/i}  - (1-t)\Delta_{/i}^*) dt \right)^{-1} - \boldsymbol{J}_{/i}^{-1}(\hat{\theta}_{/i}) \right](\boldsymbol{x}_{i}+\boldsymbol{x}_{n+1})\right| \notag\\&+\left|(\boldsymbol{x}_{i}+\boldsymbol{x}_{n+1})^\top\left[\boldsymbol{J}_{/i}^{-1}(\hat{\theta}_{/i})-\boldsymbol{J}_{/i}^{-1}(\hat{\theta}_{/i}  - \Delta_{/i}^*)  \right](\boldsymbol{x}_{i}+\boldsymbol{x}_{n+1}) \right|\notag \\
    \leq^1&\left|(\boldsymbol{x}_{i}+\boldsymbol{x}_{n+1})^\top\left[ \boldsymbol{J}_{/i}^{-1}(\hat{\theta}_{/i} ) -\left( \boldsymbol{J}_{/i}(\hat{\theta}_{/i} ) + \bar{X}_{/i}^\top \, \mathrm{diag}\!\left[\int_0^1 \gamma_{ {-(1-t)\Delta_{/i}^*} \, /i } (\hat{\theta}_{/i} ) dt \right] \bar{X}_{/i}\right)^{-1}\right](\boldsymbol{x}_{i}+\boldsymbol{x}_{n+1})\right| \notag\\
    &+\left|(\boldsymbol{x}_{i}+\boldsymbol{x}_{n+1})^\top\left[ \boldsymbol{J}_{/i}^{-1}(\hat{\theta}_{/i} ) -\left(\boldsymbol{J}_{/i}(\hat{\theta}_{/i} )
+ \bar{X}_{/i}^\top \, \mathrm{diag}\!\left[ \gamma_{ {-\Delta_{/i}^*} \, /i }(\hat{\theta}_{/i} ) \right] \bar{X}_{/i}\right)^{-1}\right](\boldsymbol{x}_{i}+\boldsymbol{x}_{n+1})\right| \notag\\
=&\left|(\boldsymbol{x}_i+\boldsymbol{x}_{n+1})^\top\Phi_1(\boldsymbol{x}_i+\boldsymbol{x}_{n+1}) \right|+\left|(\boldsymbol{x}_i+\boldsymbol{x}_{n+1})^\top\Phi_2(\boldsymbol{x}_i+\boldsymbol{x}_{n+1}) \right|,
\end{align}
where $\leq^1$ is due to (\ref{aux_definition}).

For the $\left|(\boldsymbol{x}_i+\boldsymbol{x}_{n+1})^\top\Phi_2(\boldsymbol{x}_i+\boldsymbol{x}_{n+1}) \right|$ term, using \citep[Lemma~25]{RadMaleki2020} we have 
\begin{align}\label{xixnphi2}
&\quad
\left|(\boldsymbol{x}_i+\boldsymbol{x}_{n+1})^\top
\Phi_2
(\boldsymbol{x}_i+\boldsymbol{x}_{n+1}) \right|
\nonumber\\
&\leq
\left(
\left\| \gamma_{ {-\Delta_{/i}^*} \, /i }(\hat{\theta}_{/i} ) \right\|_2
+
\left( \frac{\bar{\omega}_{\max,i}}{\nu} \right)
\left\| \gamma_{ {-\Delta_{/i}^*} \, /i }(\hat{\theta}_{/i} ) \right\|_2^2
\right)
\left\| \bar{X}_{/i} \boldsymbol{J}_{/i}^{-1}(\hat{\theta}_{/i} ) (\boldsymbol{x}_i
+\boldsymbol{x}_{n+1}) \right\|_4^2\nonumber
\\
&\le 
\left(
\left\| \gamma_{ {-\Delta_{/i}^*} \, /i }(\hat{\theta}_{/i} ) \right\|_2
+
\left( \frac{\bar{\omega}_{\max,i}}{\nu} \right)
\left\| \gamma_{ {-\Delta_{/i}^*} \, /i }(\hat{\theta}_{/i} ) \right\|_2^2
\right)
\left(\left\| \bar{X}_{/i} \boldsymbol{J}_{/i}^{-1}(\hat{\theta}_{/i} ) \boldsymbol{x}_i\right\|_4
+
\left\| \bar{X}_{/i} \boldsymbol{J}_{/i}^{-1}(\hat{\theta}_{/i} ) 
\boldsymbol{x}_{n+1}\right\|_4
\right)^2\nonumber
\\
&\le 
2\left(
\left\| \gamma_{ {-\Delta_{/i}^*} \, /i }(\hat{\theta}_{/i} ) \right\|_2
+
\left( \frac{\bar{\omega}_{\max,i}}{\nu} \right)
\left\| \gamma_{ {-\Delta_{/i}^*} \, /i }(\hat{\theta}_{/i} ) \right\|_2^2
\right)
\left\| \bar{X}_{/i} \boldsymbol{J}_{/i}^{-1}(\hat{\theta}_{/i} ) \boldsymbol{x}_i \right\|_4^2
\nonumber\\
&\quad+
2\left(
\left\| \gamma_{ {-\Delta_{/i}^*} \, /i }(\hat{\theta}_{/i} ) \right\|_2
+
\left( \frac{\bar{\omega}_{\max,i}}{\nu} \right)
\left\| \gamma_{ {-\Delta_{/i}^*} \, /i }(\hat{\theta}_{/i} ) \right\|_2^2
\right)
\left\| \bar{X}_{/i} \boldsymbol{J}_{/i}^{-1}(\hat{\theta}_{/i} ) \boldsymbol{x}_{n+1} \right\|_4^2.
    \end{align}
Similarly, for the term $\left|(\boldsymbol{x}_i+\boldsymbol{x}_{n+1})^\top\Phi_1(\boldsymbol{x}_i+\boldsymbol{x}_{n+1}) \right|$, we can prove the following bound
\begin{align}\label{xixnphi1}
    &\left|(\boldsymbol{x}_i+\boldsymbol{x}_{n+1})^\top\Phi_1(\boldsymbol{x}_i+\boldsymbol{x}_{n+1}) \right|\notag\\
    \leq&2\left(
\left\| \int_0^1 \gamma_{ {-(1-t)\Delta_{/i}^*} \, /i }(\hat{\theta}_{/i} ) \, dt \right\|_2
+
\left( \frac{\bar{\omega}_{\max,i}}{\nu} \right)
\left\| \int_0^1 \gamma_{ {-(1-t)\Delta_{/i}^*} \, /i }(\hat{\theta}_{/i} ) \, dt\right\|_2^2
\right)
\left\| \bar{X}_{/i} \boldsymbol{J}_{/i}^{-1}(\hat{\theta}_{/i} ) \boldsymbol{x}_i \right\|_4^2
\nonumber\\
&+
2\left(
\left\| \int_0^1 \gamma_{ {-(1-t)\Delta_{/i}^*} \, /i }(\hat{\theta}_{/i} ) \, dt \right\|_2
+
\left( \frac{\bar{\omega}_{\max,i}}{\nu} \right)
\left\| \int_0^1 \gamma_{ {-(1-t)\Delta_{/i}^*} \, /i }(\hat{\theta}_{/i} ) \, dt \right\|_2^2
\right)
\left\| \bar{X}_{/i} \boldsymbol{J}_{/i}^{-1}(\hat{\theta}_{/i} ) \boldsymbol{x}_{n+1} \right\|_4^2.
\end{align}
Under Assumption~\ref{assumption5}, we have
    \begin{align}\label{c2_1}
        \left\| \gamma_{ {-\Delta_{/i}^*} \, /i }(\hat{\theta}_{/i} ) \right\|_2
&\leq
\bigl\| \ddot{\boldsymbol{\ell}}_{/i}(\hat{\theta}_{/i}  - \Delta^*_{/i}) - \ddot{\boldsymbol{\ell}}_{/i}(\hat{\theta}_{/i} ) \bigr\|_2
+ \bigl\| \lambda \bigl( \ddot{r}_s(\hat{\theta}_{/i}  - \Delta^*_{/i}) - \ddot{r}_s(\hat{\theta}_{/i} ) \bigr) \bigr\|_2 \nonumber \\
&\leq 2 c_2(n) \, \bigl\| \Delta^*_{/i} \bigr\|_2 .
    \end{align}
    Likewise,
    \begin{align}\label{c2_2}
\Biggl\| \int_{0}^{1} 
\gamma_{-(1-t)\Delta^*_{/i}/i}(\hat{\theta}_{/i} ) \, dt 
\Biggr\|_2
&\leq \int_{0}^{1} 
\bigl\| \gamma_{-(1-t)\Delta^*_{/i}/i}(\hat{\theta}_{/i} ) \bigr\|_2 \, dt \nonumber \\
&\leq \int_{0}^{1} 
\bigl\| \ddot{\boldsymbol{\ell}}_{/i}(\hat{\theta}_{/i}  - (1-t)\Delta^*_{/i})
     - \ddot{\boldsymbol{\ell}}_{/i}(\hat{\theta}_{/i} ) \bigr\|_2 \, dt \nonumber \\&+ \int_{0}^{1} 
\bigl\| \lambda \bigl( \ddot{r}_s(\hat{\theta}_{/i}  - (1-t)\Delta^*_{/i})
     - \ddot{r}_s(\hat{\theta}_{/i} ) \bigr) \bigr\|_2 \, dt \nonumber \\
&\leq 2 c_2(n) \, \bigl\| \Delta^*_{/i} \bigr\|_2 .
\end{align}
    \par
    By the result in Section~\ref{eq136}, we have
    \begin{align}\label{c1_1}
        \bigl\| \Delta^*_{/i} \bigr\|_2 
\;\leq\; 
\left( \frac{\bigl| \dot{\ell}_i(\hat{\theta} ) \bigr|}{\nu} \right) 
\bigl\| \boldsymbol{x}_i \bigr\|_2 \leq\frac{c_1(n)}{\nu} \|\boldsymbol{x}_i\|_2.
    \end{align}
Using results in (\ref{c2_1}), (\ref{c2_2}) and (\ref{c1_1}), we obtain
\begin{align}\label{newa1_result}
    |A_1|\leq&\frac{1}{2}|\dot{\ell}_i(\hat{\theta} )|\cdot\left|(\boldsymbol{x}_i+\boldsymbol{x}_{n+1})^\top\Phi_1(\boldsymbol{x}_i+\boldsymbol{x}_{n+1}) \right|+\frac{1}{2}|\dot{\ell}_i(\hat{\theta} )|\cdot\left|(\boldsymbol{x}_i+\boldsymbol{x}_{n+1})^\top\Phi_2(\boldsymbol{x}_i+\boldsymbol{x}_{n+1}) \right|\notag\\
    \leq&|\dot{\ell}_i(\hat{\theta} )|\cdot\left(
\left\| \gamma_{ {-\Delta_{/i}^*} \, /i }(\hat{\theta}_{/i} ) \right\|_2
+
\left( \frac{\bar{\omega}_{\max,i}}{\nu} \right)
\left\| \gamma_{ {-\Delta_{/i}^*} \, /i }(\hat{\theta}_{/i} ) \right\|_2^2
\right)
\left\| \bar{X}_{/i} \boldsymbol{J}_{/i}^{-1}(\hat{\theta}_{/i} ) \boldsymbol{x}_i \right\|_4^2
\nonumber\\
&+|\dot{\ell}_i(\hat{\theta} )|\cdot\left(
\left\| \gamma_{ {-\Delta_{/i}^*} \, /i }(\hat{\theta}_{/i} ) \right\|_2
+
\left( \frac{\bar{\omega}_{\max,i}}{\nu} \right)
\left\| \gamma_{ {-\Delta_{/i}^*} \, /i }(\hat{\theta}_{/i} ) \right\|_2^2
\right)
\left\| \bar{X}_{/i} \boldsymbol{J}_{/i}^{-1}(\hat{\theta}_{/i} ) \boldsymbol{x}_{n+1} \right\|_4^2\notag\\
&+|\dot{\ell}_i(\hat{\theta} )|\cdot\left(
\left\| \int_0^1 \gamma_{ {-(1-t)\Delta_{/i}^*} \, /i }(\hat{\theta}_{/i} ) \, dt \right\|_2
+
\left( \frac{\bar{\omega}_{\max,i}}{\nu} \right)
\left\| \int_0^1 \gamma_{ {-(1-t)\Delta_{/i}^*} \, /i }(\hat{\theta}_{/i} ) \, dt\right\|_2^2
\right)
\left\| \bar{X}_{/i} \boldsymbol{J}_{/i}^{-1}(\hat{\theta}_{/i} ) \boldsymbol{x}_i \right\|_4^2
\nonumber\\
&+|\dot{\ell}_i(\hat{\theta} )|\cdot\left(
\left\| \int_0^1 \gamma_{ {-(1-t)\Delta_{/i}^*} \, /i }(\hat{\theta}_{/i} ) \, dt \right\|_2
+
\left( \frac{\bar{\omega}_{\max,i}}{\nu} \right)
\left\| \int_0^1 \gamma_{ {-(1-t)\Delta_{/i}^*} \, /i }(\hat{\theta}_{/i} ) \, dt \right\|_2^2
\right)
\left\| \bar{X}_{/i} \boldsymbol{J}_{/i}^{-1}(\hat{\theta}_{/i} ) \boldsymbol{x}_{n+1} \right\|_4^2\notag\\
\leq&4\|\boldsymbol{x}_{i}\|_2 
\left( \frac{c_1^2(n) c_2(n)}{\nu} \right) 
\left( 1 + \frac{2c_1(n) c_2(n)}{\nu^2} (1 + \omega_{\max,i}) \|\boldsymbol{x}_{i}\|_2 \right) 
\left\| \bar{X}_{/i} \boldsymbol{J}_{/i}^{-1}(\hat{\theta}_{/i} ) \boldsymbol{x}_{n+1} \right\|_4^2\notag\\
      &+4\|\boldsymbol{x}_{i}\|_2 \left( \frac{c_1^2(n) c_2(n)}{\nu} \right) \left( 1 + \frac{2c_1(n) c_2(n)}{\nu^2} (1 + \omega_{\max,i}) \|\boldsymbol{x}_{i}\|_2 \right) \left\| \bar{X}_{/i} \boldsymbol{J}_{/i}^{-1}(\hat{\theta}_{/i} ) \boldsymbol{x}_{i} \right\|_4^2\notag\\
\leq&4\|\boldsymbol{x}_{i}\|_2 
\left( \frac{c_1^2(n) c_2(n)}{\nu} \right) 
\left( 1 + \frac{2c_1(n) c_2(n)}{\nu^2} (1 + \omega_{\max,i}) \|\boldsymbol{x}_{i}\|_2 \right) 
\sqrt{ \bigl\| X_{/i} \boldsymbol{J}_{/i}^{-1}(\hat{\theta}_{/i} ) \boldsymbol{x}_{n+1} \bigr\|_4^4 
      + \bigl\| \boldsymbol{J}_{/i}^{-1}(\hat{\theta}_{/i} ) \boldsymbol{x}_{n+1} \bigr\|_4^4 }\notag\\
      &+4\|\boldsymbol{x}_{i}\|_2 \left( \frac{c_1^2(n) c_2(n)}{\nu} \right) \left( 1 + \frac{2c_1(n) c_2(n)}{\nu^2} (1 + \omega_{\max,i}) \|\boldsymbol{x}_{i}\|_2 \right) \sqrt{ \bigl\| X_{/i} \boldsymbol{J}_{/i}^{-1}(\hat{\theta}_{/i} ) \boldsymbol{x}_{i} \bigr\|_4^4 
      + \bigl\| \boldsymbol{J}_{/i}^{-1}(\hat{\theta}_{/i} ) \boldsymbol{x}_{i} \bigr\|_4^4 }\notag\\
      \leq& \bar{C}_{i}\left[\bigl\| X_{/i} \boldsymbol{J}_{/i}^{-1}(\hat{\theta}_{/i} ) \boldsymbol{x}_{n+1} \bigr\|_4^2 
      + \bigl\| \boldsymbol{J}_{/i}^{-1}(\hat{\theta}_{/i} ) \boldsymbol{x}_{n+1} \bigr\|_4^2 \right]+\bar{C}_{i}\left[\bigl\| X_{/i} \boldsymbol{J}_{/i}^{-1}(\hat{\theta}_{/i} ) \boldsymbol{x}_{i} \bigr\|_4^2 
      + \bigl\| \boldsymbol{J}_{/i}^{-1}(\hat{\theta}_{/i} ) \boldsymbol{x}_{i} \bigr\|_4^2 \right],
\end{align}
where $\omega_{\max,i}:=\sigma_{\max}(X_{/i}X_{/i}^\top)$, $\omega_{\max}:=\sigma_{\max}(XX^\top)$ and $\bar{\omega}_{\max,i}
:= \sigma_{\max}\!\left(\bar{X}_{/i}\bar{X}_{/i}^\top\right)$. By this definition, we can derive that
\begin{align}
\bar{\omega}_{\max,i}
&= \sigma_{\max}\!\left(\bar{X}_{/i}\bar{X}_{/i}^\top\right)
= \sigma_{\max}\!\left(\bar{X}_{/i}^\top\bar{X}_{/i}\right)
= \sigma_{\max}\!\left(
\begin{bmatrix}
X_{/i}\\
I
\end{bmatrix}^\top
\begin{bmatrix}
X_{/i}\\
I
\end{bmatrix}
\right) \notag\\[4pt]
&= \sigma_{\max}\!\left(I + X_{/i}^\top X_{/i}\right)
\le 1 + \sigma_{\max}\!\left(X_{/i}^\top X_{/i}\right)
= 1 + \omega_{\max,i},
\end{align}

For $|A_2|$ and $|A_3|$, we can use a similar method to show that 
\begin{align}\label{newa2_result}
    |A_2|\leq\frac{1}{2}\bar{C}_{i}\left[\bigl\| X_{/i} \boldsymbol{J}_{/i}^{-1}(\hat{\theta}_{/i} ) \boldsymbol{x}_{n+1} \bigr\|_4^2 
      + \bigl\| \boldsymbol{J}_{/i}^{-1}(\hat{\theta}_{/i} ) \boldsymbol{x}_{n+1} \bigr\|_4^2 \right]
\end{align}
and
\begin{align}\label{newa3_result}
    |A_3|\leq\frac{1}{2}\bar{C}_{i}\left[\bigl\| X_{/i} \boldsymbol{J}_{/i}^{-1}(\hat{\theta}_{/i} ) \boldsymbol{x}_{i} \bigr\|_4^2 
      + \bigl\| \boldsymbol{J}_{/i}^{-1}(\hat{\theta}_{/i} ) \boldsymbol{x}_{i} \bigr\|_4^2 \right].
\end{align}
Combining the results in (\ref{newa1_result}), (\ref{newa2_result}) and (\ref{newa3_result}), we obtain
\begin{align}
    |\boldsymbol{x}_{n+1}^\top\tilde{\theta}_{/i} -\boldsymbol{x}_{n+1}^\top\hat{\theta}_{/i} |\leq&\frac{3}{2}\bar{C}_{i}\left[\bigl\| X_{/i} \boldsymbol{J}_{/i}^{-1}(\hat{\theta}_{/i} ) \boldsymbol{x}_{i} \bigr\|_4^2 
      + \bigl\| \boldsymbol{J}_{/i}^{-1}(\hat{\theta}_{/i} ) \boldsymbol{x}_{i} \bigr\|_4^2 \right]\nonumber
      \\
      +&\frac{3}{2}\bar{C}_{i}\left[\bigl\| X_{/i} \boldsymbol{J}_{/i}^{-1}(\hat{\theta}_{/i} ) \boldsymbol{x}_{n+1} \bigr\|_4^2 
      + \bigl\| \boldsymbol{J}_{/i}^{-1}(\hat{\theta}_{/i} ) \boldsymbol{x}_{n+1} \bigr\|_4^2 \right]\nonumber\\
      \leq&3\bar{C}_{i}\cdot\max_{j\in\{i,n+1\}}\left[\bigl\| X_{/i} \boldsymbol{J}_{/i}^{-1}(\hat{\theta}_{/i} ) \boldsymbol{x}_{j} \bigr\|_4^2 
      + \bigl\| \boldsymbol{J}_{/i}^{-1}(\hat{\theta}_{/i} ) \boldsymbol{x}_{j} \bigr\|_4^2 \right].
\end{align}
Next, 
we define the event
\[
G:=\left\{\max_{1\leq i\leq n}|\boldsymbol{x}_{n+1}^\top\hat{\theta}_{/i} -\boldsymbol{x}_{n+1}^\top\tilde{\theta}_{/i} |\geq C\frac{\log(p)}{\sqrt{p}}\right\}.
\]
\par
Adopting the reasoning on pages 54–55 of \citep{RadMaleki2020}, where $S$ denotes the event that Assumption~3 in our paper holds (trivially we have $\mathbb{P}(S^c)\leq q_n+\tilde{q}_n$), and $q_n$ and $\tilde{q}_n$ follow the definitions in Assumption~3, we have
\begin{align}
    \mathbb{P}[G]&\leq\mathbb{P}[G|S]+\mathbb{P}[S^c]\nonumber
    \\
    &\leq\mathbb{P}[\max_{1\leq i\leq n}(\max_{j\in\{i,n+1\}}|\mathbf{x}_{j}^\top\hat{\theta}_{/i} -\mathbf{x}_{j}^\top\tilde{\theta}_{/i} |>C\frac{\log(p)}{\sqrt{p}}|S]+q_n+\tilde{q}_n\nonumber
    \\
    &\leq\frac{\mathbb{P}[\max_{1\leq i\leq n}|\boldsymbol{x}_{i}^\top\hat{\theta}_{/i} -\boldsymbol{x}_{i}^\top\tilde{\theta}_{/i} |>C\frac{\log(p)}{\sqrt{p}}]+\mathbb{P}[\max_{1\leq i\leq n}|\boldsymbol{x}_{n+1}^\top\hat{\theta}_{/i} -\boldsymbol{x}_{n+1}^\top\tilde{\theta}_{/i} |>C\frac{\log(p)}{\sqrt{p}}]}{1-q_n-\tilde{q}_n}+q_n+\tilde{q}_n\nonumber
    \\
    &\leq\frac{\mathbb{P}[3\bar{C}_{i}\cdot \left[ \bigl\| X_{/i} \boldsymbol{J}_{/i}^{-1}(\hat{\theta}_{/i} ) \boldsymbol{x}_{i} \bigr\|_4^2
      + \bigl\| \boldsymbol{J}_{/i}^{-1}(\hat{\theta}_{/i} ) \boldsymbol{x}_{i} \bigr\|_4^2\right]>C\frac{\log(p)}{\sqrt{p}}]}{1-q_n-\tilde{q}_n}\nonumber
      \\
      &\quad+\frac{\mathbb{P}[3\bar{C}_{i}\cdot \left[ \bigl\| X_{/i} \boldsymbol{J}_{/i}^{-1}(\hat{\theta}_{/i} ) \boldsymbol{x}_{n+1} \bigr\|_4^2
      + \bigl\| \boldsymbol{J}_{/i}^{-1}(\hat{\theta}_{/i} ) \boldsymbol{x}_{n+1} \bigr\|_4^2\right]>C\frac{\log(p)}{\sqrt{p}}]}{1-q_n-\tilde{q}_n}+q_n+\tilde{q}_n.
\end{align}
For $\mathbb{P}[3\bar{C}_{i}\cdot \left[ \bigl\| X_{/i} \boldsymbol{J}_{/i}^{-1}(\hat{\theta}_{/i} ) \boldsymbol{x}_{i} \bigr\|_4^2
      + \bigl\| \boldsymbol{J}_{/i}^{-1}(\hat{\theta}_{/i} ) \boldsymbol{x}_{i} \bigr\|_4^2\right]>C\frac{\log(p)}{\sqrt{p}}]$ and
      \\$\mathbb{P}[3\bar{C}_{i}\cdot \left[ \bigl\| X_{/i} \boldsymbol{J}_{/i}^{-1}(\hat{\theta}_{/i} ) \boldsymbol{x}_{n+1} \bigr\|_4^2
      + \bigl\| \boldsymbol{J}_{/i}^{-1}(\hat{\theta}_{/i} ) \boldsymbol{x}_{n+1} \bigr\|_4^2\right]>C\frac{\log(p)}{\sqrt{p}}]$, we have 
\begin{align}
    &\mathbb{P}[3\bar{C}_{i}\cdot \left[ \bigl\| X_{/i} \boldsymbol{J}_{/i}^{-1}(\hat{\theta}_{/i} ) \boldsymbol{x}_{i} \bigr\|_4^2
      + \bigl\| \boldsymbol{J}_{/i}^{-1}(\hat{\theta}_{/i} ) \boldsymbol{x}_{i} \bigr\|_4^2\right]>C\frac{\log(p)}{\sqrt{p}}]\leq\sum_{i=1}^n\mathbb{P}[\tilde{E}_i];
    \\
    &\mathbb{P}[3\bar{C}_{i}\cdot \left[ \bigl\| X_{/i} \boldsymbol{J}_{/i}^{-1}(\hat{\theta}_{/i} ) \boldsymbol{x}_{n+1} \bigr\|_4^2
      + \bigl\| \boldsymbol{J}_{/i}^{-1}(\hat{\theta}_{/i} ) \boldsymbol{x}_{n+1} \bigr\|_4^2\right]>C\frac{\log(p)}{\sqrt{p}}]\leq\sum_{i=1}^n\mathbb{P}[\tilde{E}_{n+1,i}].
\end{align}
For simplicity, we define 
\begin{align}
C_i &:= 
6(1+c)\left( \frac{\rho_{\max}}{\nu^2} \right)
\left( 1 + \omega_{\max} \sqrt{\tfrac{n-1}{p} \, \tfrac{\log(n-1)}{\log p}} \right),\nonumber
\\
F_{n+1,i} &:=
\left\{3 \,\bar{C}_i \left( 
\left\| X_{/i} \boldsymbol{J}_{/i}^{-1} \boldsymbol{x}_{n+1} \right\|_4^2
+ \left\| \boldsymbol{J}_{/i}^{-1} \boldsymbol{x}_{n+1} \right\|_4^2
\right) > \bar{C}_i C_i \sqrt{p} \log p \,\right\},\nonumber
\\
K_i &:=
\left\{ \,\frac{C}{\sqrt{p}} \geq \bar{C}_i C_i \sqrt{p} \,\right\}, \nonumber
\\
W_{n+1} &:=
\left\{ \,\|\boldsymbol{x}_{n+1}\|_2^2 > 5p \rho_{\max} \,\right\}
\;\cup\;
\left\{ \,\omega_{\max} > (\sqrt{n} + 3\sqrt{p})^2 \rho_{\max} \,\right\}.\nonumber
\end{align}
\par
Hence, we now obtain an upper bound for $\mathbb{P}[\tilde{E}_{n+1,i}]$
\begin{align}
\mathbb{P}[\tilde{E}_{n+1,i}]
&\leq \mathbb{P}[\tilde{E}_{n+1,i} \mid K_i] + \mathbb{P}[K_i^c] \nonumber \\
&\leq \mathbb{P}[F_{n+1,i} \mid K_i] + \mathbb{P}[K_i^c]
   \;\leq\; \frac{\mathbb{P}(F_{n+1,i})}{\mathbb{P}(K_i)} + \mathbb{P}[K_i^c] \nonumber \\
&\leq \frac{\mathbb{P}\!\left[ \bigl\| X_{/i} \boldsymbol{J}_{/i}^{-1} \boldsymbol{x}_{n+1} \bigr\|_4^2 
    > 2(1+c)\left(\tfrac{\rho_{\max}}{\nu^2}\omega_{\max}\right) 
    \sqrt{n-1}\,\log(n-1) \right]}{\mathbb{P}(K_i)} \nonumber \\
&\quad + \frac{\mathbb{P}\!\left[ \bigl\| \boldsymbol{J}_{/i}^{-1} \boldsymbol{x}_{n+1} \bigr\|_4^2 
    > 2(1+c)\left(\tfrac{\rho_{\max}}{\nu^2}\right) \sqrt{p}\log p \right]}{\mathbb{P}(K_i)}
    + \mathbb{P}[K_i^c] \nonumber \\
&\leq^1 \left( \frac{2}{(n-1)^c} + \frac{2}{p^c} \right) \frac{1}{\mathbb{P}(K_i)} + \mathbb{P}[K_i^c],
\end{align}
where ${\leq}^1$ is due to (\ref{lemma_27_2}).
To bound $\mathbb{P}[K_i^c]$ we define
\[
C := 
96 \sqrt{5} \left( 
\frac{c_1^2(n)c_2(n)\,(p \rho_{\max})^{3/2}}{\nu^3}
\right)
\left( 1 + \left(\sqrt{\tfrac{n}{p}} + 3\right)^2 p \rho_{\max}
   \sqrt{\tfrac{n-1}{p}\,\tfrac{\log(n-1)}{\log p}} \right)
\]
\[
\times \left( 1 + \frac{2 c_1(n) c_2(n) \sqrt{5} 
   \left( 1 + \left(\sqrt{\tfrac{n}{p}} + 3\right)^2 p \rho_{\max}\right) 
   \sqrt{p \rho_{\max}}}{\nu^2} \right).
\]
obtained by setting $c=3$, and computing $p\bar{C}_i C_i$ after putting 
$\sqrt{5p\rho_{\max}}$ and $(\sqrt{n/p} + 3\sqrt{p})^2\rho_{\max}$ bounds 
in event $W_{n+1}$, into $\|\boldsymbol{x}_{n+1}\|_2$ and $\omega_{\max}$, respectively. Next,
\begin{align}
    \mathbb{P}[K_i^c] 
&= \mathbb{P}\!\left[ \frac{C}{p} < \bar{C}_i C_i \right] 
   \;\leq\; \mathbb{P}\!\left[ C < p \bar{C}_i C_i \,\big|\, W_{n+1}^c \right] + \mathbb{P}[W_{n+1}] \nonumber \\
&= \mathbb{P}[C < C] + \mathbb{P}[W_{n+1}] \;=\; \mathbb{P}[W_{n+1}] .
\end{align}
The term $\mathbb{P}[W_{n+1}]$ is exponentially small because $\boldsymbol{x}_{n+1} \sim N(0,\Sigma)$ 
with $\rho_{\max} = \sigma_{\max}(\Sigma)$, leading to
\[
\mathbb{P}[W_{n+1}] 
\;\leq\; \mathbb{P}\!\left[ \|\boldsymbol{x}_{n+1}\|_2^2 > 5p\rho_{\max} \right] 
+ \mathbb{P}\!\left[ \sigma_{\max}(XX^\top) > (\sqrt{n}+3\sqrt{p})^2 \rho_{\max} \right]
\;\leq\; 2e^{-p},
\]
due to result in Section~\ref{lemma11} and Section~\ref{lemma12}. 
In summary, since for $p \geq 1$ we have $\tfrac{1}{1-e^{-p}} < 2$, for $c=3$ we obtain
\[
\mathbb{P}[\tilde{E}_{n+1,i}] 
\;\leq\; \frac{4}{(n-1)^3} + \frac{4}{p^3} + 2e^{-p}.
\]
\par
Since $\boldsymbol{x}_{i}$ and $\boldsymbol{x}_{n+1}$ follow same distribution, that is $\boldsymbol{x}_i\sim N(0,\Sigma)$, the same argument implies that
\[
\mathbb{P}[\tilde{E}_{i}]\leq\frac{4}{(n-1)^3}+\frac{4}{p^3}+2e^{-p}.
\]
The detailed proof for the bound of $\mathbb{P}[\tilde{E}_{i}]$ can also be found in \citep{RadMaleki2020}.
\par
As a result, we have
\begin{align}
    \mathbb{P}[G]&\leq\frac{1}{1-q_n-\tilde{q}_n}\left(\frac{8n}{(n-1)^3}+\frac{8n}{p^3}+4ne^{-p}\right)+q_n+\tilde{q}_n\nonumber
    \\
    &\leq\left(\frac{16n}{(n-1)^3}+\frac{16n}{p^3}+8ne^{-p}\right)+q_n+\tilde{q}_n.
\end{align}
It is clear that $\left[\left(\frac{16n}{(n-1)^3}+\frac{16n}{p^3}+8ne^{-p}\right)+q_n+\tilde{q}_n\right]\to0$ when $n\to\infty$.
\par
Thus, we obtain a result similar to the result in Section~\ref{Theorem3maleki}. That is, with probability at least $1-\left(\frac{16n}{(n-1)^3}+\frac{16n}{p^3}+8ne^{-p}\right)-q_n-\tilde{q}_n$, the following bound is valid
\[
\max_{1\leq i\leq n}|\boldsymbol{x}_{n+1}^\top\tilde{\theta}_{/i} -\boldsymbol{x}_{n+1}^\top\hat{\theta}_{/i} |\leq\frac{C_0}{\sqrt{p}},
\]
where $C_0$ is defined as follows
\begin{align}
&C_0 := \; 
\left( \frac{216 c^{3/2}}{\nu^{3}} \right)
\left( 1 + \sqrt{\delta_{0}}(\sqrt{\delta_{0}}+3)^{2} 
       \frac{c \log n}{\log p} \right)\left( c_{1}^{2}(n)c_{2}(n) + c_{1}^{3}(n)c_{2}^{2}(n) \,
       \frac{5\big(c^{1/2} + c^{3/2}(\sqrt{\delta_{0}}+3)^{2}\big)}{\nu^{2}} 
\right).
\end{align}

Note that in the presentation of Theorem~\ref{Theorem.1}, we replaced respectively 
$96\sqrt{5}$ and $2\sqrt{3}$ with the upper-bounds $216$ and $5$, 
and we replaced 
$\sqrt{\frac{n-1}{p} \frac{\log(n-1)}{\log p}}$
with the upper bound 
$\sqrt{\frac{n}{p} \frac{\log n}{\log p}}$.
We also used $\delta_0$ to denote $n/p$.
\section{Proof of Theorem~\ref{Theorem.2}}\label{theorem2appendix}
Using result Section~\ref{Theorem3maleki}, with probability at least $1-p_2(n)$, the following holds
\begin{align}\label{prof_thm2_1}
    \max_{1\leq i\leq n}|\boldsymbol{x}_{i}^\top\tilde{\theta}_{/i} -\boldsymbol{x}_{i}^\top\hat{\theta}_{/i} |\leq\frac{Q_0}{\sqrt{p}}.
\end{align}
Then with probability at least $1-p_2(n)$, following holds
\begin{align}\label{prof_thm2_3}
    \max_{1\leq i\leq n}|\boldsymbol{x}_{i}^\top\tilde{\theta}_{/i} -\boldsymbol{x}_{i}^\top\hat{\theta}_{/i}|=o(1).
\end{align}
Plugging (\ref{prof_thm2_3}) into $\left|R_{i}^{LOO}-\tilde{R}_{i}^{LOO}\right|$, using Assumption~\ref{assumption5}, with probability at least $1-p_2(n)$ the following result holds
    \begin{align}\label{prof_thm2_4}
        \left|R_{i}^{LOO}-\tilde{R}_{i}^{LOO}\right|=&\left|(|y_i-f(\boldsymbol{x}_i^\top\hat{\theta}_{/i})|-|y_i-f(\boldsymbol{x}_i^\top\tilde{\theta}_{/i})|)\right|\notag
        \\
        \leq&\left|f(\boldsymbol{x}_i^\top\hat{\theta}_{/i})-f(\boldsymbol{x}_i^\top\tilde{\theta}_{/i})\right|\notag
        \\
        \leq& L\cdot\max_{1\leq i\leq n}|\boldsymbol{x}_i^\top\hat{\theta}_{/i}-\boldsymbol{x}_i^\top\tilde{\theta}_{/i}|\notag
        \\
        =&o(1).
    \end{align}
Using Theorem~\ref{Theorem.1} and Assumption~\ref{assumption5}, with probability at least $1-p_1(n)$ the following result holds
\begin{align}\label{prof_thm2_5}
    \left|f(\boldsymbol{x}_{n+1}^\top\hat{\theta}_{/i})-f(\boldsymbol{x}_{n+1}^\top\tilde{\theta}_{/i})\right|\leq L\cdot\left|\boldsymbol{x}_{n+1}^\top\tilde{\theta}_{/i} -\boldsymbol{x}_{n+1}^\top\hat{\theta}_{/i}\right|=o(1).
\end{align}
Then with probability at least $(1-p_1(n)-p_2(n))$, the following results hold for $i=1,2,...,n$:
\begin{align}\label{prof_thm2_6}
    \left|(f(\boldsymbol{x}_{n+1}^\top\tilde{\theta}_{/i})-\tilde{R}_i^{LOO})-(f(\boldsymbol{x}_{n+1}^\top\hat{\theta}_{/i})-R_i^{LOO})\right|=o(1);
\end{align}
\begin{align}\label{prof_thm2_7}
    \left|(f(\boldsymbol{x}_{n+1}^\top\tilde{\theta}_{/i})+\tilde{R}_i^{LOO})-(f(\boldsymbol{x}_{n+1}^\top\hat{\theta}_{/i})+R_i^{LOO})\right|=o(1).
\end{align}
With (\ref{prof_thm2_6}) and (\ref{prof_thm2_7}), using Lemma~\ref{lemma.2}, we can easily derive that
    \begin{align}\label{prof_thm2_8}
        \hat{q}_{n,\alpha}^-\{f(\boldsymbol{x}_{n+1}^\top\tilde{\theta}_{/i})-\tilde{R}_i^{LOO}\}=\hat{q}_{n,\alpha}^-\{f(\boldsymbol{x}_{n+1}^\top\hat{\theta}_{/i})-R_i^{LOO}\}+o(1);
    \end{align}
    \begin{align}\label{prof_thm2_9}
        \hat{q}_{n,\alpha}^+\{f(\boldsymbol{x}_{n+1}^\top\tilde{\theta}_{/i})+\tilde{R}_i^{LOO}\}=\hat{q}_{n,\alpha}^+\{f(\boldsymbol{x}_{n+1}^\top\hat{\theta}_{/i})+R_i^{LOO}\}+o(1).
    \end{align}
    That is, with probability at least $(1-p_1(n)-p_2(n))$, the following results hold
    \begin{align}
        \tilde{U}_{jk+}=U_{jk+}+o(1),\;\tilde{L}_{jk+}=L_{jk+}+o(1).
    \end{align}
    We can utilize similar method to conclude that, with probability at least $(1-p_1(n)-p_2(n))$, the following results hold
\begin{align}
\tilde{U}_{jkm}=U_{jkm}+o(1),\;\tilde{L}_{jkm}=L_{jkm}+o(1).
\end{align}
Theorem~\ref{Theorem.2} has been proved.
\section{Proof of Theorem~\ref{Theorem.3} \& \ref{Theorem.4}}\label{theorem3appendix}
Recall that in Assumption~\ref{assumption5} we defined that
\[
p_1(n)=\left(\frac{16n}{(n-1)^3}+\frac{16n}{p^3}+8ne^{-p}\right)+q_n+\tilde{q}_n
\]
and
\[
p_2(n)=\left(\frac{8n}{(n-1)^3}+\frac{8n}{p^3}+4ne^{-p}\right)+q_n+\tilde{q}_n.
\]
Under Assumption~\ref{assumption1} to Assumption~\ref{assumption5}, and the condition that $n,p\to\infty$, $n/p=\delta_0\in(0,\infty)$, it is clear that
\[
p_1(n)\xrightarrow{\;\text{$n\to\infty$}\;} 0,\quad  p_2(n)\xrightarrow{\;\text{$n\to\infty$}\;} 0.
\]
    Using (\ref{prof_thm2_8}), (\ref{prof_thm2_9}) and Assumption~\ref{assumption5}, we obtain that, with probability at least $(1-p_1(n)-p_2(n))$, the following holds
    \begin{align}
        \mathbb{P}\{y_{n+1}\in\tilde{C}_{n,\alpha}^{Jackknife+}\}=\mathbb{P}\{y_{n+1}\in\hat{C}_{n,\alpha}^{Jackknife+}\}+L\cdot o(1)=\mathbb{P}\{y_{n+1}\in\hat{C}_{n,\alpha}^{Jackknife+}\}+o(1).
    \end{align}
    According to \citep{jackknife+}, we have
    \begin{align}\label{orginal_conclusion_jk+_1}
    \hat{C}_{n,\alpha}^{Jackknife+}=[\hat{q}_{n,\alpha}^{-}\{f(\boldsymbol{x}_{n+1}^\top\hat{\theta}_{/i})-R_i^{LOO}\},\hat{q}_{n,\alpha}^{+}\{f(\boldsymbol{x}_{n+1}^\top\hat{\theta}_{/i})+R_i^{LOO}\}]
    \end{align}
    and
    \begin{align}\label{orginal_conclusion_jk+_2}
        \mathbb{P}\{y_{n+1}\in\hat{C}_{n,\alpha}^{Jackknife+}\}\geq 1-2\alpha.
    \end{align}
    \par
    As s result, using (\ref{orginal_conclusion_jk+_2}), we have
        \begin{align}
            \mathbb{P}\{y_{n+1}\in\tilde{C}_{n,\alpha}^{Jackknife+}\}\geq (1-2\alpha-o(1))\cdot(1-p_1(n)-p_2(n))=1-2\alpha-o(1).
        \end{align}
    \par
    Theorem~\ref{Theorem.3} has been proved.
    \par
    Proof of Theorem~\ref{Theorem.4} is similar to Theorem~\ref{Theorem.3}.
\section{Proof of lemmas}\label{prooflemmas}
\subsection{Proof of Lemma~\ref{lemma.2}}
Lemma~\ref{lemma.2} follows from the following more precise version:
\begin{lemma}\label{lem2}
    Let \(\{a_i\}_{i=1}^n\) be a sorted sequence such that \(a_1 \le a_2 \le \dots \le a_n\), and suppose there exist sequences \(\{b_i\}_{i=1}^n\) satisfying \(|b_i - a_i| \le  \varepsilon\), for some $\varepsilon>0$ and all \(i = 1, \dots, n\). Then for any \(q \in \{1,2,...,n\}\) we have
\[
|b_{(q)}-a_q|\le \epsilon,
\]
where $b_{(q)}$ is the $q$-th smallest value of $\{b_i\}_{i=1}^n$.
\end{lemma}
\begin{proof}
Under the constraint of $b_i-a_i\le \epsilon$,
the maximum of $b_{(q)}$ is achieved when $b_i=a_i+\epsilon$ for each $i$, in which case $b_{(q)}=a_q+\epsilon$. 
Hence $b_{(q)}\le a_q+\epsilon$.
Similarly, $b_{(q)}\ge a_q-\epsilon$, and the proof is completed.
\end{proof}

\section{Proof of Theorem~\ref{thm6}}\label{theorem6appendix}
The proof of Theorem~\ref{thm6} relies on some auxiliary results (for simplicity, we define $\hat{\mu}_{/i}(X)=X^\top\hat{\beta}_{/i}$ and $\hat{\mu}(X)=X^\top\hat{\beta}$):
\begin{lemma}\label{lem3}
Under assumptions C1-C5, we have the following (in the sense of weak convergence)
\begin{align}
\max_{1\le i\le n}|\hat{\mu}_{/i}(X_{n+1})-\hat{\mu}(X_{n+1})|=o(1).
\label{e_37}
\end{align} 
\end{lemma}
\begin{proof}
Recall that \citep[Theorem~2.2]{ElKaroui2018} showed that 
\begin{align}
\max_{1\le i\le n}\|\hat{\beta}-\hat{\beta}_{/i}-\eta_i\|_2
=O_{L_2}\left(\frac{\polylog(n)}{n}\right),
\label{e38}
\end{align}
where the notations are explained as follows:
$O_{L_2}(\cdot)$ denotes a bound on a random variable in the square root of the second moment, 
and $\hat{\beta}_{/i}+\eta_i$
can be thought of as an approximate formula for $\hat{\beta}$ using the leave-one-out estimator $\hat{\beta}_{/i}$.
The correction term
\begin{align}
\eta_i := \frac{1}{n} (S_i + \tau I)^{-1} X_i 
\rho' \!\left( \prox_{c_i \rho_i}(r_{i,(i)}) \right),
\label{e37}
\end{align}
where we defined 
\begin{align}
c_i &:= \frac{1}{n} X_i^{\top} (S_i + \tau I)^{-1} X_i; 
\\
S_i& := \frac{1}{n} \sum_{j \neq i} \rho''(r_{j,(i)}) X_j X_j^{\top},
\end{align}
where 
\begin{align}
r_{j,(i)}:=y_j-X_j^{\top}\hat{\beta}_{/j}
\end{align}
denotes the leave-one-out residual (in particular,  $r_{i,(i)}=R^{LOO}_i$).

Noting that $S_i$, $X_i$, and $X_{n+1}$ are independent,
and $X_i$ and $X_{n+1}$ are i.i.d.\ with zero mean and identity covariance, 
we see that 
$\frac{1}{n}X_{n+1}^{\top} (S_i + \tau I)^{-1} X_i$ has zero mean and $O(\frac1{n})$ second moment.
Under the assumption of $\|\rho'\|_{\infty}<\infty$,
this in turn implies, by \eqref{e37}, that $X_{n+1}^{\top}\eta_i=o_{L_2}(1)$ (vanishing in the sense of second moment).
Then using the fact that $\|X_{n+1}\|_2=O_{L_2}(\sqrt{n})$ and 
\eqref{e38} we obtain $\max_{1\le i\le n}\|X_{n+1}^{\top}(\hat{\beta}-\hat{\beta}_{/i})\|_2
=o_{L_2}(1)$,
which is equivalent to \eqref{e_37}.
\end{proof}

Next, recall that $\hat{\mu}^y(X_i):=X_i^{\top}\hat{\beta}^y$, $i=1,\dots,n+1$, where $\hat{\beta}^y$ denotes the output of the regression algorithm trained on the augmented dataset $\{(X_1,y_1),\dots,(X_n,y_n),(X_{n+1},y)\}$.
Also, adopt the notations
\begin{align}
R_i&:=y_i-X_i^{\top}\hat{\beta},
\quad i=1,\dots, n,
\\
\psi(t)&:=\rho'(t), \quad t\in\mathbb{R}.
\end{align}
We then have:
\begin{lemma}
For any given $M>0$, we have the following (in the sense of convergence in probability):
\begin{align}
    \max_{y\in [-M,M]\cap\{0,\pm\frac1{M},\pm\frac{2}{M},\dots\}}\max_{1\le i\le n}|\hat{\mu}^y(X_i)-\hat{\mu}(X_i)|
    &=o(1);
    \label{e43}
    \\
    \max_{y\in [-M,M]\cap\{0,\pm\frac1{M},\pm\frac{2}{M},\dots\}}\max_{1\le i\le n}
    \left|\hat{\mu}_{/i}(X_i)+c_i\psi(\prox_{c_i \rho_i}(r_{i,(i)}))-\hat{\mu}(X_i)\right|
    &=o(1).
    \label{e54}
\end{align}
\end{lemma}
\begin{proof}
The proof of \eqref{e43} is almost identical to Lemma~\ref{lem3} (by viewing $\hat{\mu}$ as the estimator leaving out the $(n+1)$-th observation, and then taking the union bound over all $y$ in the set $[-M,M]\cap\{0,\pm\frac1{M},\pm\frac{2}{M},\dots\}$), which relies on the error bound on the leave-one-out approximation formula \eqref{e38}.
It still remains, however, to check whether the proof depends on how $y_{n+1}=y$ is selected.
Below, we provide details of justification that \eqref{e_37} continues to hold if $y_i$ is a fixed value in $\mathbb{R}$ (rather than randomly generated from the stochastic model), which is equivalent to \eqref{e43} (by switching the role of $n+1$ and $i$).
This, in turn, relies on two properties:
\begin{description}
\item[a)] \eqref{e38} continues to hold when $y_i$ is an arbitrary deterministic quantity.
\item[b)] $X_{n+1}^{\top}\eta_i=o_{L_2}(1)$ continues to hold when $y_i$ is an arbitrary deterministic quantity.
\end{description}
It is easy to see that b) is true,
under our assumption of $\|\rho'\|_{\infty}<\infty$,
since it is still true that
$\frac{1}{n}X_{n+1}^{\top} (S_i + \tau I)^{-1} X_i$ has zero mean and $O(\frac1{n})$ second moment.
For a), we revisit the original proof \citep[p122-p125]{ElKaroui2018},
which is mostly based on several deterministic non-asymptotic bounds therein, except \citep[Lemma~3.7]{ElKaroui2018}, which relies on convergence in high probability.
However, \citep[Lemma~3.7]{ElKaroui2018} uses nothing but a high probability bound that $\sup_{j\neq i}\frac{\|(S_i+\tau I)^{-1}X_j\|}{\tau\sqrt{n}}=O(1)$,
the independence of $\frac{\|(S_i+\tau I)^{-1}X_j\|}{\tau\sqrt{n}}$ and $X_i$,
and sub-Gaussianity of $X_i$,
which in turn implies by union bound that $\sup_{j\neq i}\frac{X_j^{\top}(S_i+\tau I)^{-1}X_i}{n}=\frac1{\sqrt{n}}\polylog(n)$ with high probability.
Thus a) also holds, and we have verified \eqref{e43}.

Next, we observe that \eqref{e54} because of \eqref{e38} and \eqref{e37}, where again we used the fact that \eqref{e38} continues to hold for arbitrary deterministic $y_i$.
\end{proof}

\begin{lemma}\label{l64}
For any given $M>0$, we have the following (in the sense of convergence in probability):
\begin{align}
    \max_{y\in [-M,M]\cap\{0,\pm\frac1{M},\pm\frac{2}{M},\dots\}}
    |\hat{\mu}(X_{n+1})
    +c\psi(\prox_{c\rho}(y-\hat{\mu}(X_{n+1})))
    -\hat{\mu}^y(X_{n+1})|
    &=o(1),
\end{align}
where we defined 
\begin{align}
c &:= \frac{1}{n+1} X_{n+1}^{\top} (S + \tau I)^{-1} X_{n+1}; 
\\
S& := \frac{1}{n+1} \sum_{j =1}^n \rho''(y-\hat{\mu}(X_{n+1})) X_j X_j^{\top}.
\end{align}
\end{lemma}
\begin{proof}
The proof is identical to the proof of \eqref{e54}, again noting that the approximation error bound (\eqref{e38}, but leaving out $n+1$ instead of $i$) continues to hold with fixed deterministic $y_{n+1}=y$.
\end{proof}

\begin{proof}[Proof of Theorem~\ref{thm6}]
Lemma~\ref{lem3} and Lemma~\ref{lem2} imply that $|U_{\rm J}-U_{\rm mm}|=o(1)$,
and 
\begin{align}
|\hat{\mu}(X_{n+1})
+\hat{q}_{n,\alpha}^+\{R^{LOO}_i\}
-U_{\rm J}|=o(1).
\end{align}
Next, recall that \citep[Appendix~5]{ElKaroui2018} shows that
$r_{i,(i)}=R^{LOO}_i$ 
 behaves like (in the sense of moment convergence)
\[
\epsilon_i +  \sqrt{\mathbb{E}\!\left( \|\hat{\beta} - \beta_0\|^2 \right)} Z_i,
\]
where $Z_i \sim \mathcal{N}(0,1)$ independent of $\epsilon_i$; 
furthermore, if $i \neq j$, $r_{i,(i)}$ and $r_{j,(j)}$ are asymptotically 
(pairwise) independent. 
Since $\lim_{p \to \infty} \|\hat{\beta}-\beta_0\|=r$,
this implies that the empirical distribution 
$(r_{i,(i)})_{i=1}^n$ converges to $\mathcal{N}(0,\sigma_{\epsilon}^2+r^2)$.
Thus 
\begin{align}
|\hat{q}_{n,\alpha}^+\{R^{LOO}_i\}-\sqrt{\sigma_{\epsilon}^2+r^2}z_{\alpha/2}|=o(1)
\label{se62}
\end{align}
and we have proved \eqref{e_al}.

It remains to prove \eqref{e_bl}.
Note that by definition, $y\in [-M,M]\cap\{0,\pm\frac1{M},\pm\frac{2}{M},\dots\}$ is included in the full conformal interval if
\begin{align}
|y-\hat{\mu}^y(X_{n+1})|\le \hat{q}^+\{R_i\}.
\end{align}
Suppose that $y_{\max}$ and $y_{\min}$ denote the maximum and minimum $y$ under this criterion.
By Lemma~\ref{l64}, we see that $y_{\max}$ satisfies 
\begin{align}
y_{\max}-\hat{\mu}(X_{n+1})-
c\psi(\prox_{c\rho}(y_{\max}-\hat{\mu}(X_{n+1})))=\hat{q}^+\{R_i\}
+e_{n,M},
\label{se63}
\end{align}
where $e_{n,M}$ is an error term satisfying $\lim_{M\to+\infty}\lim_{n\to\infty}e_{n,M}=0$.
By \eqref{e38}, we have 
\begin{align}
R_i-R_i^{LOO}=-c_i\psi(R_i)+o(1),
\end{align}
implying $R_i=\prox_{c_i\rho}(R_i^{LOO}+o(1))$.
Furthermore, using the Sherman–Morrison formula we can show that $c_i=c+o(1)$, so that $R_i=\prox_{c\rho}(R_i^{LOO}+o(1))$.
Since $\prox_{c\rho}$ is an increasing function, we have 
\begin{align}
\hat{q}^+\{R_i\}
&=
\hat{q}^+\{\prox_{c\rho}(R_i^{LOO}+o(1))\}
\\
&=
\prox_{c\rho}(\hat{q}^+\{R_i^{LOO}+o(1)\}).
\end{align}
Then \eqref{se63} shows that 
\begin{align}
\prox_{c\rho}
(y_{\max}-\hat{\mu}(X_{n+1}))
=
\prox_{c\rho}(\hat{q}^+\{R_i^{LOO}+o(1)\})+e_{n,M}.
\end{align}
Since we assumed $\|\rho''\|_{\infty}<\infty$, and $c=O(1)$ with high probability,  $\prox_{c\rho}^{-1}$ must be Lipschitz with high probability. 
Hence 
\begin{align}
y_{\max}-\hat{\mu}(X_{n+1})
=
\hat{q}^+\{R_i^{LOO}+o(1)\}
+\tilde{e}_{n,M}
\end{align}
for some $\tilde{e}_{n,M}$ satisfying $\lim_{M\to\infty}\lim_{n\to\infty}\tilde{e}_{n,M}=0$. 
Then \eqref{e_bl} follows
by \eqref{se62}.
\end{proof}
\end{document}